\documentclass[letterpaper]{article}
\usepackage[top=0.75in, bottom=1.0in, left=1.0in, right=1.0in]{geometry}

\usepackage{graphicx}
\usepackage{amsfonts}
\usepackage{amsmath}
\usepackage{subfigure}
\usepackage{algorithmic}
\usepackage{algorithm}
\usepackage{booktabs}
\usepackage{amsthm}
\usepackage{authblk}

\newcommand{\argmax}{\operatornamewithlimits{\arg \max}}
\newcommand{\argmin}{\operatornamewithlimits{\arg \min}}

\newcommand{\rset}{{\textrm{set}}}
\newcommand{\Var}{{\textrm{Var}}}

\newcommand{\bc}{{\mathbf{c}}}

\newcommand{\bp}{{\mathbf{p}}}
\newcommand{\bw}{{\mathbf{w}}}
\newcommand{\bx}{{\mathbf{x}}}
\newcommand{\bxbar}{{\bar{\bx}}}
\newcommand{\bX}{{\mathbf{X}}}
\newcommand{\bXtilde}{{\tilde{\bX}}}
\newcommand{\bXbar}{{\bar{\bX}}}
\newcommand{\by}{{\mathbf{y}}}
\newcommand{\bybar}{{\bar{\by}}}
\newcommand{\bY}{{\mathbf{Y}}}
\newcommand{\bYtilde}{{\tilde{\bY}}}
\newcommand{\bYbar}{{\bar{\bY}}}

\newcommand{\bmu}{{\boldsymbol \mu}}

\newcommand{\bbE}{{\mathbb{E}}}
\newcommand{\bbR}{{\mathbb{R}}}
\newcommand{\bbZ}{{\mathbb{Z}}}

\newcommand{\calN}{{\mathcal{N}}}
\newcommand{\calO}{{\mathcal{O}}}
\newcommand{\calP}{{\mathcal{P}}}

\newcommand{\calX}{{\mathcal{X}}}
\newcommand{\calXset}{{\calX_{\rset}}}

\newcommand{\sI}{{\mathsf{I}}}
\newcommand{\sK}{{\mathsf{K}}}
\newcommand{\sKtilde}{{\tilde{\sK}}}

\newcommand{\frakX}{{\mathfrak{X}}}

\newcommand{\binommL}{{\binom{m}{L}}}
\newcommand{\kset}{{k_{\rset}}}
\newcommand{\ksettilde}{{\tilde{k}_{\rset}}}

\newcommand{\secref}[1]{Section~\ref{#1}}
\newcommand{\figref}[1]{Fig.~\ref{#1}}
\newcommand{\tabref}[1]{Table~\ref{#1}}
\newcommand{\algref}[1]{Algorithm~\ref{#1}}
\newcommand{\corref}[1]{Corollary~\ref{#1}}
\newcommand{\thmref}[1]{Theorem~\ref{#1}}
\newcommand{\lemref}[1]{Lemma~\ref{#1}}

\newtheorem{theorem}{Theorem}
\newtheorem{lemma}{Lemma}
\newtheorem{property}{Property}
\newtheorem{corollary}{Corollary}

\begin{document}

\title{Bayesian Optimization with Approximate Set Kernels}

\author[1]{Jungtaek Kim}
\author[2]{Michael McCourt}
\author[1]{Tackgeun You}
\author[3]{Saehoon Kim}
\author[4]{Seungjin Choi}

\affil[1]{{\normalsize POSTECH, Pohang, Republic of Korea}}
\affil[2]{{\normalsize SigOpt, San Francisco, USA}}
\affil[3]{{\normalsize Kakao Brain, Seongnam, Republic of Korea}}
\affil[4]{{\normalsize BARO AI, Seoul, Republic of Korea}}
\affil[ ]{{\normalsize \texttt{jtkim@postech.ac.kr}}}

\maketitle

\begin{abstract}
We propose a practical Bayesian optimization method over sets, 
to minimize a black-box function that takes a set as a single input.
Because set inputs are permutation-invariant, 
traditional Gaussian process-based Bayesian 
optimization strategies which assume vector inputs can fall short.
To address this, we develop a Bayesian optimization method 
with \emph{set kernel} that is used to build surrogate functions.
This kernel accumulates similarity over set elements to enforce 
permutation-invariance, but this comes at a greater computational cost.
To reduce this burden, we propose two key components: 
(i) a more efficient approximate set kernel which is still positive-definite 
and is an unbiased estimator of the true set kernel 
with upper-bounded variance in terms of the number of subsamples, 
(ii) a constrained acquisition function optimization over sets, 
which uses symmetry of the feasible region that defines a set input.
Finally, we present several numerical experiments which demonstrate 
that our method outperforms other methods.
\end{abstract}

\section{Introduction\label{sec:intro}}

Bayesian optimization is an effective method to optimize
an expensive black-box function.
It has proven useful in several applications, including
hyperparameter optimization~\cite{SnoekJ2012neurips,HutterF2011lion},
neural architecture search~\cite{ZophB2016iclr,KandasamyK2018neurips},
material design~\cite{FrazierP2016ismdd,HaghanifarS2019mh}, and synthetic
gene design~\cite{GonzalezJ2014bayesopt}.
Classic Bayesian optimization assumes a search
region $\calX \subset \bbR^d$
and a black-box function $f$ evaluated
in the presence of additive noise $\epsilon$,
i.e., $y = f(\bx) + \epsilon$ for $\bx \in \calX$.

Unlike this standard Bayesian optimization formulation, 
we assume that a search region is
$\calXset=\{\{\bx_1, \ldots, \bx_m\} \;|\; \bx_i \in \calX \subset \bbR^d\}$
for a fixed positive integer $m$.
Thus, for $\bX \in \calXset$, $f$ would take in a \emph{set} containing $m$ elements,
all of length $d$, and return a noisy function value $y$:
\begin{equation}
    y = f \left( \bX \right) + \epsilon.
    \label{eqn:y_f_X}
\end{equation}

Our motivating example comes from the soft $k$-means clustering algorithm 
over a dataset $\calP = \{\bp_1, \ldots, \bp_N\}$;
in particular, we aim to find the optimal initialization of such an algorithm.
The objective function for this problem is a squared loss
function which takes in the cluster initialization points
$\{\bx_1, \ldots, \bx_k\}$ and returns the weighted distance
between the points in $\calP$ and the converged cluster
centers $\{\bc_1, \ldots, \bc_k\}$.
See \cite{LloydS1982ieeetit} for more details.

Some previous research has attempted to build Gaussian process (GP) models
on set data.
\cite{GarnettR2010ipsn} proposes a method over discrete sets 
using stationary kernels over the first Wasserstein distance between two sets,
though the power set of fixed discrete sets as domain space is not our interest.
However, this method needs the complexity $\calO(n^2 m^3 d)$ 
to compute a covariance matrix with respect to $n$ sets.
Moreover, since it only considers stationary kernels,
GP regression is restricted to a form that cannot express non-stationary models~\cite{PaciorekC2004neurips}.

Therefore, we instead adapt and augment a strategy proposed in \cite{GartnerT2002icml}
involving the creation of a specific \emph{set kernel}.
This set kernel uses a kernel defined on the elements $\bx \in \bbR^d$ of the sets to build up
its own sense of covariance between sets.
In turn, then, it can be directly used to build surrogate functions through
GP regression, which can power the Bayesian optimization strategy, by \lemref{lem:set_pd}.

A key contribution of this article is the development of a
computationally efficient approximation to this set kernel.
Given $n$ total observed function values, the cost of constructing
the matrix required for fitting the GP is $\calO(n^2m^2d)$ 
where $m \geq n$ (see the complexity analysis in \secref{subsec:boos}).
The approximate set kernel proposed in this work 
uses random subsampling to reduce the computational cost to 
$\calO(n^2L^2d)$ for $L < m$ while still producing an unbiased
estimate of the expected value of the true kernel.

Another primary contribution is a constrained acquisition function optimization over set inputs.
The next query set to observe is found by optimizing the acquisition function defined on a set $\bX \in \calXset$.
Using the symmetry of the space, this function can be efficiently optimized 
with a rejection sampling.
Furthermore, we provide a theoretical analysis on cumulative regret bounds of our framework, 
which guarantees the convergence quality in terms of iterations.

\section{Background\label{sec:back}}

In this section, we briefly introduce previous studies, notations, and related work necessary to
understand our algorithm.

\subsection{Bayesian Optimization}

Bayesian optimization seeks to minimize an unknown function $f$
which is expensive to evaluate, 
$\bx^\star = \argmin_{\bx \in \calX} f(\bx)$, 
where $\calX \subset \bbR^d$ is a compact space.
It is a sequential optimization strategy
which, at each iteration, performs
the following three computations:
\begin{enumerate}
    \item Using the $n$ data presently available,
        $\{(\bx_i, y_i)\}$ for $i \in [n]$,
        build a probabilistic surrogate model $s_n$ meant to approximate $f$.
    \item Using the surrogate model $s_n$, compute an acquisition
        function $a_n$, which represents the utility of next acquiring
        data at some new point $\bx$.
    \item Observe $y_{n+1}$ from a true function $f$ at the location $\bx_{n+1}=\argmax_{\bx \in \calX} a_n(\bx)$.
\end{enumerate}

After exhausting a predefined budget $T$,
Bayesian optimization returns the best point, $\bx^\dagger$, that has the minimum observation.
The benefit of this process is that the optimization of the expensive
function $f$ has been replaced by the optimization of much cheaper
and better understood acquisition functions $a_n$.

In this paper, we use GP regression~\cite{RasmussenCE2006book} to produce the surrogate function $s_n$;
from $s_n$, we use the Gaussian process upper confidence bound (GP-UCB) criterion~\cite{SrinivasN2010icml}:
$a_n(\bx) = - \mu_n(\bx) + \beta_n \sigma_n(\bx)$,
where $\mu_n(\cdot)$ and $\sigma_n^2(\cdot)$ are posterior mean and variance functions computed by $s_n$,  
and $\beta_n$ is a trade-off hyperparameter for exploration and exploitation at iteration $n$.
See \cite{BrochuE2010arxiv,ShahriariB2016procieee,FrazierPI2018arxiv}
for the details.

\subsection{Set Kernel}

We introduce the notation required for performing
kernel approximation of functions on sets. A set of $m$ vectors
is denoted as $\bX = \{\bx_1, \ldots, \bx_m\}$ ,
where $\bx_i$ is in a compact space $\calX \subset \bbR^d$.
In a collection of $n$ such
sets (as will occur in the Bayesian optimization setting), the $k$th
set would be denoted $\bX^{(k)} = \{\bx_1^{(k)}, \ldots, \bx_m^{(k)}\}$.
Note that we are restricting all sets to be of the same size
$|\bX^{(k)}| = m$.\footnote{In principle, sets of varying size can be considered,
but we restrict to same sized sets to simplify our analysis.}

To build a GP surrogate, we require a prior belief of
the covariance between elements in
$\calXset = \{\{\bx_1, \ldots, \bx_m\} \;|\; \bx_i \in \calX \subset \bbR^d\}$.
This belief is imposed in
the form of a positive-definite kernel
$\kset: \calXset \times \calXset \to \bbR$; see
\cite{ScholkopfB2002book,FasshauerGE2015ws}
for more discussion on approximation with kernels.
In addition to the symmetry $\kset(\bX^{(i)}, \bX^{(j)}) = \kset(\bX^{(j)}, \bX^{(i)})$
required in standard kernel settings, kernels on sets require an additional property:
the ordering of elements in $\bX$ should be immaterial 
(since sets have no inherent ordering).

Given an empirical approximation of the kernel mean 
$\bmu_{\bX} \approx |\bX|^{-1} \sum_{i = 1}^{|\bX|} \phi(\bx_i)$, 
where $\phi$ is a feature map $\bbR^d \to \bbR^{d'}$ 
and $d'$ is a dimensionality of projected space by $\phi$,
a set kernel~\cite{GartnerT2002icml,MuandetK2017ftml} is defined as
\begin{equation}
    \kset \left( \bX^{(1)}, \bX^{(2)} \right) = \left\langle
        \bmu_{\bX^{(1)}}, \bmu_{\bX^{(2)}} \right\rangle
    = \frac{1}{|\bX^{(1)}| |\bX^{(2)}|}
        \sum_{i = 1}^{|\bX^{(1)}|}
        \sum_{j = 1}^{|\bX^{(2)}|}
        k \left( \bx_i^{(1)}, \bx_j^{(2)} \right),
	\label{eqn:set_kernel}
\end{equation}
where $k(\bx_i, \bx_j) = \langle \phi(\bx_i), \phi(\bx_j) \rangle = \phi(\bx_i)^\top \phi(\bx_j)$.
Here, $k: \calX \times \calX \to \bbR$ is a positive-definite kernel defined 
to measure the covariance between the $d$-dimensional elements of the sets
(e.g., a squared exponential or Mat\'ern kernel).
The kernel \eqref{eqn:set_kernel} arises when comparing
a class of functions on different probability measures
with the intent of understanding if the measures might be
equal~\cite{GrettonA2012jmlr}.

\begin{figure}[t]
	\centering
	\includegraphics[width=0.98\textwidth]{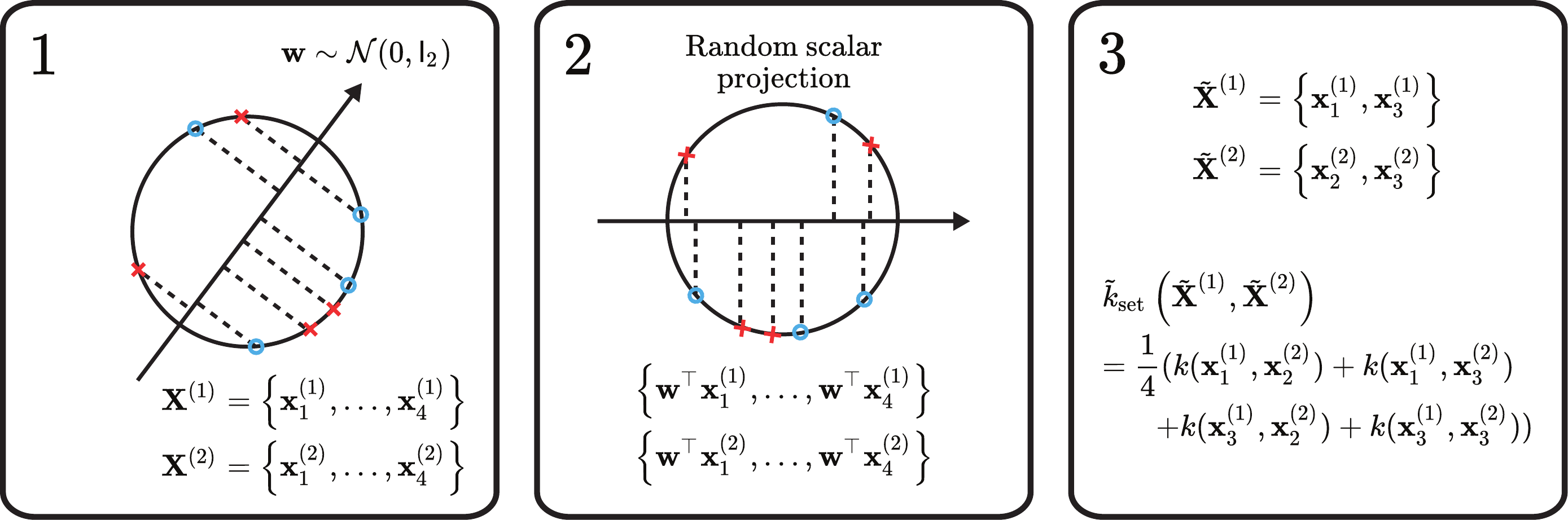}
	\caption{Illustration that shows how to select $L$ instances from sets, which originally have $m$ instances.
	In this example, $m = 4$ and $L = 2$.
	\emph{(Phase 1)} Two set inputs are projected onto a vector that is randomly drawn from the standard Gaussian distribution.
	The points that have same color belong to same set (e.g., blue and red).
	\emph{(Phase 2)} The location of the projections onto the line determines the order of the instances.
	\emph{(Phase 3)} Using the order of instances, two instances uniformly sampled are selected and they are used to compute the approximate set kernel value.}
	\label{fig:mmd_approx}
\end{figure}

\subsection{Related Work}

Although it has been raised in different interests, 
meta-learning approaches dealt with set inputs are promising in a machine learning community,
because they can generalize distinct tasks 
with meta-learners~\cite{EdwardsH2017iclr,ZaheerM2017neurips,FinnC2017icml,GarneloM2018icml}.
In particular, 
they propose feed-forward
neural networks which take permutation-invariant and variable-length
inputs: they have the goal of obtaining features derived from the sets
with which to input to a standard (meta-)learning routine.
Since they consider modeling of set structure, they are related to our work, 
but they are interested in their own specific examples 
such as point cloud classification, few-shot learning,
and image completion.

In Bayesian optimization literature, 
\cite{GarnettR2010ipsn} suggests a method to find a set 
that produces a global minimum with respect to
discrete sets, each of which is an element of power set of entire set.
This approach solves the problem related to set structure using the first Wasserstein distance over sets.
However, for the reason why the time complexity of the first Wasserstein distance is $\calO(n^2 m^3 d)$,
they assume a small cardinality of sets and discrete search space for the global optimization method.
Furthermore, their method restricts the number of iterations for optimizing an acquisition function, 
since the number of iterations should increase exponentially due to the curse of dimensionality.
This implies that finding the global optimum of acquisition function is hard to achieve.

Compared to \cite{GarnettR2010ipsn}, we consider continuous domain space 
which implies an acquired set is composed of any instances in a compact space $\calX$.
We thus freely use off-the-shelf global optimization method 
or local optimization method~\cite{ShahriariB2016procieee}
with relatively large number of instances in sets.
In addition, its structure of kernel 
is $k_{\textrm{st}}(d(\bX^{(1)}, \bX^{(2)}))$
where $k_{\textrm{st}}(\cdot)$ is a stationary kernel~\cite{GentonMG2001jmlr} 
and $d(\cdot, \cdot)$ is a distance function over two sets 
(e.g., in \cite{GarnettR2010ipsn} the first Wasserstein distance).
Using the method proposed in \secref{sec:method},
a non-stationary kernel might be considered in modeling a surrogate function.

Recently, \cite{BuathongP2020aistats} solves a similar set optimization problem using Bayesian optimization with deep embedding kernels.\footnote{Although this work~\cite{BuathongP2020aistats} refers to our preliminary non-archival presentation~\cite{KimJ2019automl}, we mention \cite{BuathongP2020aistats} here due to a close relationship with this work and its importance.}
Compared to our method, it employs a kernel over RKHS embeddings as a kernel for set inputs, 
and shows its strict positive definiteness.

\section{Proposed Method\label{sec:method}}

We first propose and analyze an approximation to the
set kernel \eqref{eqn:set_kernel} for GP regression in this section.
Then, we present a Bayesian optimization framework over sets, 
by introducing our Bayesian optimization with approximate set kernels 
and a constrained optimization method for finding the next set to evaluate.

In order for \eqref{eqn:set_kernel} to be a viable kernel 
of a GP regression, it must be positive-definite.
To discuss this topic, we denote a list
of $n$ sets with the notation
$\frakX = [ \bX^{(1)}, \ldots, \bX^{(n)} ] \in \calXset^n$; 
in this notation, the order of the entries matters.

\begin{lemma}
	\label{lem:set_pd}
	Suppose we have a list $\frakX$ which contains distinct sets
	$\bX^{(i)}$ for $i \in [n]$.  We define the matrix
	$\sK \in \bbR^{n \times n}$ as
	\begin{equation}
	    (\sK)_{ij} = \kset \left( \bX^{(i)}, \bX^{(j)} \right),
	    \label{eqn:kernel_set_all}
	\end{equation}
	for $\kset$ defined with
	a chosen inner kernel $k$ as in \eqref{eqn:set_kernel}.
	Then, $\sK$ is a symmetric positive-semidefinite matrix
	if $k$ is a symmetric positive-definite kernel.
\end{lemma}

This proof appears in \cite[Lemma 1]{HausslerD1999tr},
and is also discussed in \cite{GartnerT2002icml}.

\subsection{Approximation of the Set Kernel\label{subsec:sk_approx_gp}}

Computing \eqref{eqn:kernel_set_all} requires pairwise comparisons 
between all sets present in $\frakX$, which has computational 
complexity $\calO(n^2 m^2 d)$.
To alleviate this cost, we propose to approximate 
\eqref{eqn:set_kernel} with 
\begin{equation}
	\ksettilde \left( \bX^{(1)}, \bX^{(2)}; \pi, \bw, L \right) =
	\kset \left( \bXtilde^{(1)}, \bXtilde^{(2)} \right),
	\label{eqn:set_kernel_approx}
\end{equation}
where $\pi: [m] \to [m]$, $\bw \in \bbR^d$ and $L \in \bbZ_+$ 
and $\bXtilde^{(i)}$ is a subset of $\bX^{(i)}$ which 
is defined by those three quantities (we omit explicitly 
listing them in $\bXtilde^{(i)}$ to ease the notation).

The goal of the approximation strategy is to convert from 
$\bX^{(i)}$ (of size $m$) to $\bXtilde^{(i)}$ (of size $L$) 
in a consistent fashion during all the $\ksettilde$ computations 
comprising $\sK$.
As shown in \figref{fig:mmd_approx}, we accomplish this in two steps:
\begin{enumerate}
    \item Use a randomly generated vector $\bw$ to 
        impose an (arbitrary) ordering of the elements 
        of all sets $\bX^{(i)}$, and 
    \item Randomly permute the indices $[m]$ via a function $\pi$.
\end{enumerate}
These random strategies are defined once before computing the $\sK$ matrix, 
and then used consistently throughout the entire computation.

\begin{algorithm}[t]
	\caption{Forming the $\sKtilde$ approximation to $\sK$\label{alg:approx}}
	\begin{algorithmic}[1]
		\REQUIRE A list of sets $\frakX = [\bX^{(1)}, \ldots, \bX^{(n)}]$,
		    $L \in \bbZ_+$.
		\ENSURE A matrix $\sKtilde$.
		\STATE Draw a random vector $\bw$ from $\calN(0, \sI_d)$.
		\STATE Create a random permutation of the integers
		$[m]$ with which to define $\pi$.
		\STATE Using $\bw$, assign the ordering of elements
		in $\bX^{(i)}$ according to \eqref{eqn:orderingofelements}.
		\STATE Using $\pi$, determine the subsets $\bXtilde^{(i)}$
		as selected according to \eqref{eqn:subsampledef}.
		\STATE Using these subsets and \eqref{eqn:set_kernel_approx},
		populate the matrix with values:
		\begin{equation*}
		    (\sKtilde)_{ij} = \ksettilde(\bX^{(i)}, \bX^{(j)}; \bw, \pi, L).
		\end{equation*}
		\STATE \textbf{return} a computed matrix $\sKtilde$
	\end{algorithmic}
\end{algorithm}

To impose an ordering of the elements, we use a \emph{random scalar projection} 
$\bw \in \bbR^d$ such that the elements of $\bw$ 
are drawn from the standard normal distribution. If the scalar 
projections of each $\bx_i$ are computed, this produces the set 
of scalar values $\{ \bw^\top \bx_1, \ldots, \bw^\top \bx_m \}$, which 
can be sorted to generate an ordered list of 
\begin{equation}
    \label{eqn:orderingofelements}
    [\ell_1, \ldots, \ell_m],\quad
    \bw^\top \bx_{\ell_1} \leq \ldots \leq \bw^\top \bx_{\ell_m},
\end{equation}
for an ordering of distinct indices $\ell_1, \ldots, \ell_m \in [m]$.
Ties between $\bw^\top \bx_i$ values can be dealt with arbitrarily.
The function $\pi$ then is simply a random bijection of the integers 
$[m]$ onto themselves. Using this, we can sample 
$L$ vectors from $\bX^{(i)}$:
\begin{equation}
	\label{eqn:subsampledef}
	\bXtilde^{(i)} = \{ \bx_{\ell_j} \;|\; \ell_j = \pi(j) \ \ \textrm{for} \ \ j \in [L] \}.
\end{equation}
This process, given $\bw$, $\pi$, and $L$, is sufficient for computing $\kset$, 
as presented in~\algref{alg:approx}.

\subsection{Properties of the Approximation\label{subsec:propertiesoftheapproximation}}

The covariance matrix for this approximation of the set kernel, 
which we denote by 
$(\sKtilde)_{ij} = \ksettilde(\bX^{(i)}, \bX^{(j)}; \bw, \pi, L)$, 
should approximate the full version of covariance matrix, 
$\sK$ from \eqref{eqn:kernel_set_all}.
Because of the random structure introduced in 
\secref{subsec:sk_approx_gp}, the matrix $\sKtilde$ will be random.
This will be addressed in 
\thmref{thm:unbiasedestimate}, but for now, $\sKtilde$ 
represents a single realization of that random variable, 
not the random variable itself.
To be viable, this approximation must satisfy 
the following requirements:

\begin{property}
	\label{pro:symmetry}
	The approximation satisfies pairwise symmetry:
	\begin{equation}
		\ksettilde \left( \bX^{(i)}, \bX^{(j)}; \bw, \pi, L \right)
		=
		\ksettilde \left( \bX^{(j)}, \bX^{(i)}; \bw, \pi, L \right).
	\end{equation}
	Since $\bXtilde^{(i)}$ is uniquely defined given
	$\bw, \pi, L$, this simplifies to
	$\kset(\bXtilde^{(i)}, \bXtilde^{(j)}) = \kset(\bXtilde^{(j)}, \bXtilde^{(i)})$, 
	which is true because $\kset$ is symmetric.
\end{property}

\begin{property}
	\label{pro:ordering}
	The ``ordering'' of the elements in the sets $\bX^{(i)}, \bX^{(j)}$
	should not matter when computing $\ksettilde$. Indeed,
	because \eqref{eqn:orderingofelements} enforces ordering
	based on $\bw$, and not whatever arbitrary indexing is
	imposed in defining the elements of the set, the kernel
	will be permutation-invariant.
\end{property}

\begin{property}
	\label{pro:positivedefinite}
	The kernel approximation \eqref{eqn:set_kernel_approx}
	reduces to computing $\kset$ on a lower cardinality
	version of the data (with $L$ elements selected
	from $m$). Because $\kset$ is positive-definite
	on these $L$-element sets, we know
	that $\ksettilde$ is also positive-definite.
\end{property}

\begin{property}
	\label{pro:computation}
	Since the approximation method aims to choose subsets of input sets,
	the computational cost becomes lower than the original formulation.
\end{property}

Missing from these four properties is a statement 
regarding the quality of the approximation. We address 
this in \thmref{thm:unbiasedestimate} and \thmref{thm:variance}, 
though we first start by stating \lemref{lem:innersummation}.

\begin{lemma}
	\label{lem:innersummation}
	Suppose there are two sets $\bX, \bY \in \calXset$.
	Without loss of generality, let $\bX^{(i)}$ and $\bY^{(j)}$ 
	denote the $i$th and $j$th of $\binommL$ possible subsets 
	containing $L$ elements 
	of $\bX$ and $\bY$, respectively, in an arbitrary ordering.
	For $L \in [m]$, 
	\begin{equation}
		\sum_{i=1}^{\binommL}\!\sum_{j=1}^{\binommL}\!\sum_{a=1}^{L}\!\sum_{b=1}^{L}{k(\bxbar^{(i)}_a, \bybar^{(j)}_b)} = \frac{L^2 \binommL^2}{m^2} \!\sum_{c=1}^{m}\!\sum_{d=1}^{m}{k(\bx_c, \by_d)},
	\end{equation}
	where $\bxbar^{(i)}_a$ and $\bybar^{(j)}_b$ are the $a$th and $b$th 
	elements of $\bX^{(i)}$ and $\bY^{(j)}$, respectively, 
	in an arbitrary ordering.
\end{lemma}

\begin{proof}
	We can rewrite the original summation in
	a slightly more convoluted form, as
	\begin{align}
		&\sum_{i=1}^{\binommL}\sum_{j=1}^{\binommL} \sum_{a=1}^{L}\sum_{b=1}^{L}{k(\bxbar^{(i)}_a, \bybar^{(j)}_b)} \nonumber\\
		&=
		\sum_{i=1}^{\binommL}\sum_{j=1}^{\binommL} 
		\sum_{a=1}^{L}\sum_{b=1}^{L}
		\sum_{c=1}^{m}\sum_{d=1}^{m}
		k(\bxbar^{(i)}_a, \bybar^{(j)}_b)
		I_{\bxbar^{(i)}_a, \bybar^{(j)}_b}(\bx_c,\by_d) \nonumber\\
		&=
		\sum_{c=1}^{m}\!\sum_{d=1}^{m}\!
		\left[
		\sum_{i=1}^{\binommL}\!\sum_{j=1}^{\binommL}\!
		\sum_{a=1}^{L}\!\sum_{b=1}^{L}
		k(\bxbar^{(i)}_a, \bybar^{(j)}_b)
		I_{\bxbar^{(i)}_a, \bybar^{(j)}_b}(\bx_c,\by_d)
		\right],
		\label{eqn:startoflemma}
	\end{align}
	where $I_{\bxbar^{(i)}_a, \bybar^{(j)}_b}(\bx_c, \bx_d) = 1$
	if $\bxbar^{(i)}_a = \bx_c$ and $\bybar^{(j)}_b = \by_d$, and
	0 otherwise.  As these are finite summations, they can be
	safely reordered.
	
	The symmetry in the structure and
	evaluation of the summation implies that
	as each $\bx_c$ quantity will be paired with each
	$\by_d$ quantity the same number of times.
	Therefore, we need only consider the number of times
	that these quantities appear.

	We recognize that this summation follows a pattern related to
	Pascal's triangle. Among the $\binommL$ possible subsets
	$\bxbar$ of $\bX$, only the fraction $L / m$ of those
	contain the quantity $\bx_c$ for all $c \in [m]$
	(irrespective of how that entry may be denoted in
	$\bxbar^{(i)}_a$ terminology). Because of the symmetry
	mentioned above, each of those $\bx_c$ quantities is
	paired with each of the $\by_d$ quantities the same
	$\frac{L}{m} \binommL$ number of times.
	This result implies that
	\begin{equation}
		\sum_{i=1}^{\binommL} \sum_{j=1}^{\binommL}
		\sum_{a=1}^{L}\sum_{b=1}^{L}
		k(\bxbar^{(i)}_a, \bybar^{(j)}_b)
		I_{\bxbar^{(i)}_a, \bybar^{(j)}_b}(\bx_c,\by_d)
		= \frac{L^2 \binommL^2}{m^2} k(\bx_c, \by_d),
		\label{eqn:sillysummation}
	\end{equation}
	where $I_{\bxbar^{(i)}_a, \bybar^{(j)}_b}(\bx_c, \by_d) = 1$
	if $\bxbar^{(i)}_a = \bx_c$ and $\bybar^{(j)}_b = \by_d$, and
	0 otherwise. Substituting \eqref{eqn:sillysummation}
	into the bracketed quantity in \eqref{eqn:startoflemma}
	above completes the proof.
\end{proof}

We start by introducing the notation $W$ and $\Pi$
to be random variables such that $W \sim \calN(0,\sI_d)$
and $\Pi$ is a uniformly random permutation of the
integers between 1 and $m$.  These are the distributions
defining the $\bw$ and $\pi$ quantities described above.
With this, we note that
$\ksettilde(\bX, \bY; W, \Pi, L)$ is a random variable.

We also introduce the notation $\sigma_L(\bX)$
to be the distribution of random subsets of $\bX$ with
$L$ elements selected without replacement, the outcome
of the subset selection from \secref{subsec:sk_approx_gp}.
This notation allows us to write the quantities
\begin{align}
	\bbE_{W, \Pi} [ \ksettilde(\bX, \bY; W, \Pi, L) ] &= \bbE_{\bXbar, \bYbar} [\kset(\bXbar, \bYbar) ] 
		\equiv
		\bbE[\kset(\bXbar, \bYbar)], \label{eqn:expectedvaluedef}\\
	\Var_{W, \Pi} [ \ksettilde(\bX, \bY; W, \Pi, L) ] &= \Var_{\bXbar, \bYbar} [\kset(\bXbar, \bYbar) ]
		\equiv
		\Var[\kset(\bXbar, \bYbar)],
	\label{eqn:variancedef}
\end{align}
for $\bXbar \sim \sigma_L(\bX), \bYbar \sim \sigma_L(\bY)$.  We
have dropped the random variables from the expectation and variance
definitions for ease of notation.

\begin{theorem}
	\label{thm:unbiasedestimate}
	Suppose that we are given two sets $\bX, \bY \in \calXset$ and
	$L \in \bbZ_+$.  Suppose, furthermore, that $\bw$ and $\pi$
	can be generated randomly as defined in
	\secref{subsec:sk_approx_gp} to form subsets
	$\bXtilde$ and $\bYtilde$.
	The value of $\ksettilde(\bX, \bY; \bw, \pi, L)$
	is an unbiased estimator of the value of
	$\kset(\bX, \bY)$.
\end{theorem}

\begin{proof}
	Our goal is to show that
	$\bbE[\kset(\bXbar, \bYbar)]\!=\!\kset(\bX, \bY)$, where
	$\bbE[\kset(\bXbar, \bYbar)]$ is defined in
	\eqref{eqn:expectedvaluedef}.
    
	We first introduce an extreme case: $L = m$.
	If $L = m$, the subsets we are constructing are the full sets,
	i.e., $\sigma_m(\bX)$ contains only one element, $\bX$.  Thus,
	$\ksettilde(\bX, \bY; W, \Pi, m)=\kset(\bX, \bY)$ is not a random
	variable.
	
	For $1 \leq L < m$, we compute this expected
	value from the definition (with some abuse of notation):
	\begin{equation}
        \bbE[\kset(\bXbar, \bYbar)]
	    = \sum_{\bXbar, \bYbar} \kset(\bXbar, \bYbar) p(\bXbar, \bYbar).
	\end{equation}
	There are $\binommL$ subsets, all of which could be
	indexed (arbitrarily) as $\bXbar^{(i)}$ for
	$1 \leq i \leq \binommL$.
	The probability mass function is uniform across all	subsets,
	meaning that
	$p(\bXbar = \bXbar^{(i)}, \bYbar = \bYbar^{(j)}) = 1 / \binom{m}{L}^2$.
	Using this, we know
	\begin{equation}
	    \bbE[\kset(\bXbar, \bYbar)]
	    = \sum_{i = 1}^{\binommL} \sum_{j=1}^{\binommL}
	        \kset(\bXbar^{(i)}, \bYbar^{(j)}) \frac{1}{\binommL^2}.
        \label{eqn:exp_kset_bXbar_bYbar}
	\end{equation}
	We apply \eqref{eqn:set_kernel} to see that
	\begin{equation}
	    \kset(\bXbar^{(i)},\bYbar^{(j)}) =
	        \frac{1}{L^2}
	        \sum_{a = 1}^L \sum_{b = 1}^L k(\bxbar^{(i)}_a, \bybar^{(j)}_b),
        \label{eqn:kset_bXbari_bXbarj}
	\end{equation}
	following the notational conventions used above. The expectation
	involves four nested summations,
	\begin{equation}
	    \bbE[\kset(\bXbar,\bYbar)]
	    =
	    \frac{1}{L^2\binommL^2}
	    \sum_{i=1}^{\binommL} \sum_{j=1}^{\binommL}
	    \sum_{a=1}^L\sum_{b=1}^L k(\bxbar^{(i)}_a, \bybar^{(j)}_b).
	\end{equation}
	We utilize \lemref{lem:innersummation} to rewrite
	this as
	\begin{equation}
        \bbE[\kset(\bXbar,\bYbar)]
        = \frac{1}{L^2\binommL^2}
    	\frac{L^2 \binommL^2}{m^2}
        \sum_{c=1}^{m} \sum_{d=1}^{m}{k(\bx_c, \by_d)} 
        = \frac{1}{m^2}
        \sum_{c=1}^{m} \sum_{d=1}^{m}{k(\bx_c, \by_d)}.
	\end{equation}
\end{proof}

\begin{theorem}
\label{thm:variance}
    Under the same conditions as in \thmref{thm:unbiasedestimate}, 
    suppose that $k(\bx, \bx')\geq0$ for all $\bx, \bx' \in \calX$.
    The variance of $\ksettilde(\bX, \bY; \bw, \pi, L)$ is bounded by a function of $m$, $L$ and $\kset(\bX, \bY)$:
    \begin{equation}
    \Var \left[ \ksettilde(\bX, \bY; \bw, \pi, L) \right] 
    \leq \left(\frac{m^4}{L^4} - 1\right)\kset(\bX, \bY)^2.
    \end{equation}
\end{theorem}

\begin{proof}
    The variance of $\kset(\bXbar, \bYbar)$, defined in
    \eqref{eqn:variancedef}, is computed as
    \begin{align}
        \Var\left[ \kset(\bXbar, \bYbar) \right] &= \bbE\left[ \left(\kset(\bXbar, \bYbar) - \bbE \left[ \kset(\bXbar,\bYbar) \right]\right)^2 \right] \nonumber \\
        &= \bbE \left[\kset(\bXbar, \bYbar)^2 \right] + \kset(\bX, \bY)^2 - 2 \kset(\bX, \bY) \bbE \left[\kset(\bXbar, \bYbar) \right] \nonumber \\
        &= \bbE \left[\kset(\bXbar, \bYbar)^2 \right] - \kset(\bX, \bY)^2,
        \label{eqn:variance_final}
    \end{align}
    where \thmref{thm:unbiasedestimate} is invoked to produce the
    final line.
    Using \eqref{eqn:exp_kset_bXbar_bYbar} and \eqref{eqn:kset_bXbari_bXbarj}, we can express the first term of \eqref{eqn:variance_final} as
    \begin{equation}
        \bbE \left[\kset(\bXbar, \bYbar)^2 \right] = \sum_{i = 1}^{\binommL} \sum_{j = 1}^{\binommL} \left(
            \frac{1}{L^2} \sum_{a = 1}^L \sum_{b = 1}^L k(\bxbar_a^{(i)}, \bxbar_b^{(j)}) \right)^2 \frac{1}{\binommL^2}.
        \label{eq:expectedsquareddef}
    \end{equation}
    At this point, we invoke the fact that $k(\bx, \bx')\geq0$ to state
    \begin{equation}
        0\leq
        \sum_{a = 1}^L \sum_{b = 1}^L k(\bxbar_a^{(i)}, \bxbar_b^{(j)}) \leq
        \sum_{a = 1}^m \sum_{b = 1}^m k(\bx_a, \bx_b),
    \end{equation}
    which is true because the summation to $m$ terms contains all of the
    elements in the summation to $L$ terms, as well as other (nonnegative)
    elements.  Using this, we can bound \eqref{eq:expectedsquareddef} by
    \begin{align}
        \bbE \left[\kset(\bXbar, \bYbar)^2 \right] &\leq \sum_{i = 1}^{\binommL} \sum_{j = 1}^{\binommL} \left(
            \frac{1}{L^2} \sum_{a = 1}^m \sum_{b = 1}^m k(\bx_a, \bx_b) \right)^2 \frac{1}{\binommL^2} \nonumber \\
        &= \frac{m^4}{\binommL^2 L^4}
            \sum_{i = 1}^{\binommL} \sum_{j = 1}^{\binommL} \left(
            \frac{1}{m^2} \sum_{a = 1}^m \sum_{b = 1}^m k(\bx_a, \bx_b) \right)^2  \nonumber \\
        &= \frac{m^4}{\binommL^2 L^4} \binommL^2 \kset(\bX, \bY)^2 = \frac{m^4}{L^4} \kset(\bX, \bY)^2.
        \label{eqn:inequality_exp_kset}
    \end{align}
    Therefore, with \eqref{eqn:inequality_exp_kset}, \eqref{eqn:variance_final} can be written as
    \begin{equation}
        \Var \left[ \kset(\bXbar, \bYbar) \right]
        \leq \frac{m^4}{L^4} \kset(\bX, \bY)^2 - \kset(\bX, \bY)^2 
        = \left( \frac{m^4}{L^4} - 1 \right) \kset(\bX, \bY)^2,
    \end{equation}
    which concludes this proof.
\end{proof}

The restriction $k(\bx, \bx') \geq 0$ is satisfied by many standard
covariance kernels (such as the Gaussian, the Mat\'ern family and the
multiquadric) as well as some more interesting choices (such as the
Wendland or Wu families of compactly supported kernels).  It does, however,
exclude some oscillatory kernels such as the Poisson kernel as well as
kernels defined implicitly which may have an oscillatory behavior.
More discussion on different types of kernels and their properties can
be found in the kernel literature~\cite{FasshauerGE2015ws}.

By \thmref{thm:variance}, we can naturally infer the fact that 
the upper bound of the variance decreases quickly, as $L$ is close to $m$.

\subsection{Bayesian Optimization over Sets\label{subsec:boos}}

For Bayesian optimization over $\calXset$, the goal is to identify 
the set $\bX \in \calXset$ such that a given function 
$f: \calXset \to \bbR$ is minimized.
As shown in \algref{alg:boos}, Bayesian optimization over 
sets follows similar steps as laid out in \secref{sec:back}, 
except that it involves the space of set inputs 
and requires a surrogate function on $\calXset$.
As we have already indicated, we plan to use a GP 
surrogate function, with prior covariance defined either 
with \eqref{eqn:set_kernel} or \eqref{eqn:set_kernel_approx} 
and a Mat\'ern 5/2 inner kernel $k$.

A GP model requires computation on the 
order of $\calO(n^3)$ at the $n$th step of the Bayesian optimization 
because the $\sK$ matrix must be inverted.
Compared to the complexity for computing a full version of 
the set kernel $\calO(n^2 m^2 d)$, the complexity of computing 
the inverse is smaller if roughly $m \geq n$ (that is, computing the 
matrix can be as costly or more costly than inverting it).
Because Bayesian optimization is efficient sampling-based 
global optimization, $n$ is small and the situation 
$m \geq n$ is reasonable.
Therefore, the computation reduction by our approximation can be effective in reducing 
complexity of all steps for Bayesian optimization over sets.

In addition, since the Cholesky decomposition, instead of matrix inverse 
is widely used to compute posterior mean and variance functions~\cite{RasmussenCE2006book}, 
the time complexity for inverting a covariance matrix can be reduced.
But still, if $m$ is relatively large, our approximation is effective.
In this paper, we compute GP surrogate 
using the Cholesky decomposition.
See the effects of $L$ 
in \secref{subsec:exp_approx}.

\begin{algorithm}[t]
	\caption{Bayesian Optimization over Sets\label{alg:boos}}
	\begin{algorithmic}[1]
		\REQUIRE A domain $\calXset$, a function $f: \calXset \to \bbR$,
			a budget $T \in \bbZ_+$.
		\ENSURE Best acquired set $\bX^\dagger$.
		\STATE Choose an initial point $\bX^{(1)}$ 
			randomly from $\calXset$ and evaluate
			$y_1 = f(\bX^{(1)}) + \epsilon_1$.
		\FOR{$k$ from $1$ to $T-1$}
			\STATE Fit the surrogate model $s_k$ to all data $\{(\bX^{(i)}, y_i)\}_{i=1}^k$.
			\STATE Compute the acquisition function $a_k$ from $s_k$.
			\STATE Identify $\bX^{(k+1)}=\argmax_{\bX \in \calXset} a_k(\bX)$.
			\STATE Evaluate $y_{k+1} = f(\bX^{(k+1)})+\epsilon_{k + 1}$.
		\ENDFOR
		\STATE \textbf{return} $\bX^\dagger = \bX^{(i)}$ if $y_i = \max_{j \in [T]} y_j$
	\end{algorithmic}
\end{algorithm}

\subsection{Acquisition Function Optimization over Sets\label{subsec:afo_sets}}

An acquisition function optimization step is one of the primary steps in Bayesian optimization,
because this step is for finding an acquired example, 
which exhibits the highest potential of the global optimum.
Compared to generic vector-input Bayesian optimization methods,
our Bayesian optimization over sets needs to find a query set $\bX$ on $\calXset$, 
optimizing the acquisition function over sets.

Since $m$ is fixed when we find $\bX$, 
off-the-shelf optimization methods such as L-BFGS-B~\cite{LiuDC1989mp}, 
DIRECT~\cite{JonesDR1993jota}, and CMA-ES~\cite{HansenN2016arxiv} can be employed
where a set is treated as a concatenated vector.
However, because of the symmetry of the space $\calXset$, 
these common optimization methods search more of the space than is required.
For instance, if we optimize the function on $\{1, 2, 3\}$, 
we would not need to consider such the sets $\{1, 3, 2\}, \{2, 1, 3\}, \ldots, \{3, 2, 1\}$.
Therefore, we suggest a constrained acquisition function optimization with rejection sampling algorithm.
The sample complexity of this method decreases by a factor of $m!$.

First of all, we sample all instances of the sets to initialize acquisition function optimization from uniform distribution.
By the rejection sampling, some of the sets sampled are rejected 
if each of them is located outside of the symmetric region.
The optimization method (i.e., CMA-ES) finds the optimal set of the acquisition function, 
starting from those initial sets selected from the aforementioned step.
In the optimization step, our optimization strategy forces every optimization result to locate in the symmetric search space.
Finally, we pick the best one among the converged sets as the next set to observe.

\subsection{Regret Analysis\label{subsec:regret}}

To produce a regret bound on Bayesian optimization with our approximation as well as one of Bayesian optimization with the set kernel,
we follow the framework to prove the regret bounds of the multi-armed bandit-like acquisition function in \cite{SrinivasN2010icml}.
To simplify the analysis, we discuss only the Mat\'ern kernel for $\nu > 2$ as a base kernel for GP.

Inspired by \cite{KandasamyK2015icml}, we assume that 
our objective function $f$ can be expressed as a summation of functions over instances 
and they can be collected to a function that can take a single vector
with the i.i.d.~property of instances in a set:
\begin{equation}
	f(\bX) = \frac{1}{m} \sum_{i = 1}^{m} g(\bx_i) = h([\bx_1, \ldots, \bx_m]),
\end{equation}
which implies that $h$ can take $md$-dimensional concatenated inputs.
Thus, we can state our Bayesian optimization with the set kernel 
follows the cumulative regret bound proposed 
by \cite{SrinivasN2010icml} and \cite{KandasamyK2015icml}.
Given $T$ available sets $\{ \bX^{(1)}, \ldots, \bX^{(T)} \}$, 
let $\delta \in (0, 1)$ and trade-off hyperparameter for GP-UCB 
$\beta_t \in \calO(md \log t)$.
A cumulative regret bound $R_T$ is 
\begin{equation}
R_T = \calO(md 2^d T^{\frac{\nu + d(d + 1)}{2\nu + d(d + 1)}}),
\end{equation}
with the probability at least $1 - \delta$,
under the mild assumptions:
(i) $\kset(\bX, \bX') \leq 1$; 
(ii) the bounded reproducing kernel Hilbert space (RKHS) norm $\|f\|_{\kset} < b$ where $b > 0$.
Moreover, $\ksettilde(\bX, \bX'; \bw, \pi, L) \leq 1$ and $\|f\|_{\ksettilde} < b$ 
are naturally satisfied by their definitions.

By \thmref{thm:unbiasedestimate} and \thmref{thm:variance}, 
we can define \corref{cor:regret}, which is related to the aforementioned regret bound.

\begin{corollary}
	\label{cor:regret}
	Suppose that $0<\delta\ll1$.
	$R_T^{(L)}$ is a cumulative regret computed by the approximate set kernel with 
	$L$ of $m$ instances in each set.
	By \thmref{thm:unbiasedestimate} and \thmref{thm:variance}, 
	$\bbE_{L} [ R_T^{(L)} ] = R_T$, 
	is satisfied with the probability at least $1 - \delta$.
\end{corollary}

\begin{proof}
	Since our approximation is an unbiased estimator of original set kernel $\kset$ with bounded variance 
	as discussed in \thmref{thm:unbiasedestimate} and \thmref{thm:variance},
	the expectation of cumulative regrets with the approximation is equal to $R_T$.
\end{proof}

\section{Experiments\label{sec:exp}}

In this section, we present various experimental results, 
to show unique applications of our method as well as the motivating problems.
First, we conduct our method on the experiments regarding set kernel approximation and constrained acquisition function optimization, 
in order to represent the effectiveness of our proposed method.
Then, we optimize two synthetic functions and clustering algorithm initialization, 
which take a set as an input.
Finally, we present the experimental results on active nearest neighbor search for point clouds.

We define the application-agnostic baseline methods, Vector and Split:

\noindent \textbf{Vector}
A standard Bayesian optimization is performed over a
$md$-dimensional space where, at the $n$th step,
the available data $\frakX_n \in \calXset^n$ is vectorized to
$[\bx_1, \ldots, \bx_n]$ for $\bx_i \in \bbR^{md}$ with
associated function values.  At each step, the
vectorized next location $\bx_{n+1}$ is converted into
a set $\bX_{n+1}$.

\noindent \textbf{Split}
Individual Bayesian optimization strategies are
executed on the $m$ components comprising $\calX$.
At the $n$th step, the available data $\frakX_n \in \calXset^n$
is decomposed into $m$ sets of data, the $i$th of which
consists of $[\bx^{(i)}_1,\ldots,\bx^{(i)}_n]$ with
associated data.  The $m$ vectors produced during each step
of the optimization are then collected to form $\bX_{n+1}$
at which to evaluate $f$.

For Vector and Split baselines, 
to satisfy the permutation-invariance property, 
we determine the order of elements in a set as the ascending order by the $l^2$ norm of the elements.

We use Gaussian process regression~\cite{RasmussenCE2006book} 
with a set kernel or Mat\'ern 5/2 kernel as a surrogate function.
Because computing the inverse of covariance matrix needs heavy computations, 
we employ the Cholesky decomposition instead~\cite{RasmussenCE2006book}.
For the experiments with set kernel, Mat\'ern 5/2 kernel is used as a base kernel.
All Gaussian process regression models are optimized 
by marginal likelihood maximization with the BFGS optimizer, to find kernel hyperparameters.
We use Gaussian process upper confidence bound criterion~\cite{SrinivasN2010icml} 
as an acquisition function for all experiments.
CMA-ES~\cite{HansenN2016arxiv} and its constrained version are applied 
in optimizing the acquisition function.
Furthermore, five initial points are given to start a single round of Bayesian optimization.
Unless otherwise specified, most of experiments are repeated 10 times.
For the results on execution time, all the results include the time consumed in evaluating a true function.
All results are run via CPU computations.
All implementations which will be released as open source project are written in \texttt{Python}.
Thanks to~\cite{PedregosaF2011jmlr}, 
we use \texttt{scikit-learn} in many parts of our implementations.

\subsection{Set Kernel Approximation\label{subsec:exp_approx}}

\begin{figure}[t]
	\centering
    \subfigure[Kernel approximation]
    {
    	\includegraphics[width=0.41\textwidth]{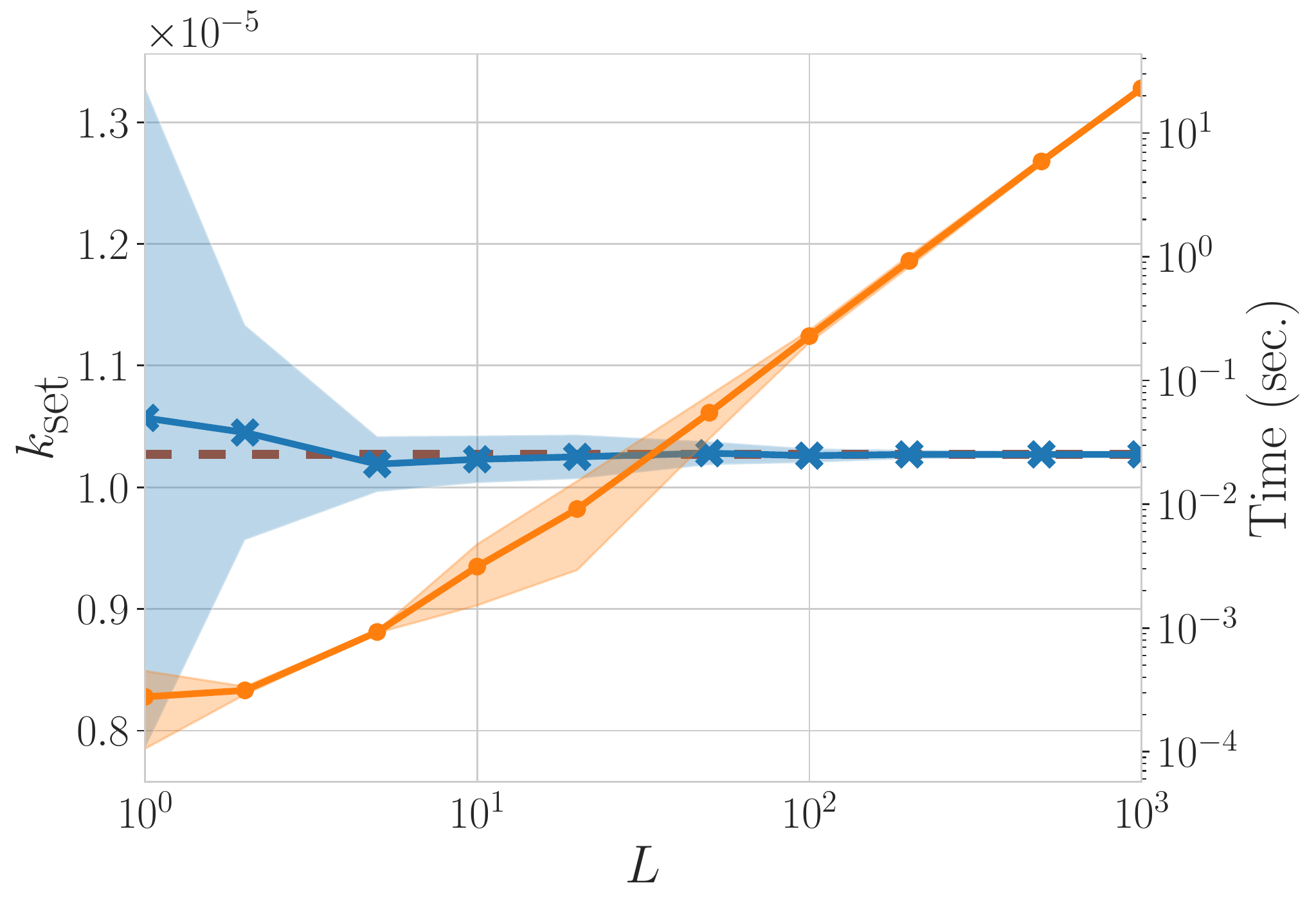}
    	\label{fig:effects_approx_L}
    }
    \quad\quad\quad
    \subfigure[Constrained optimization]
    {
    	\includegraphics[width=0.37\textwidth]{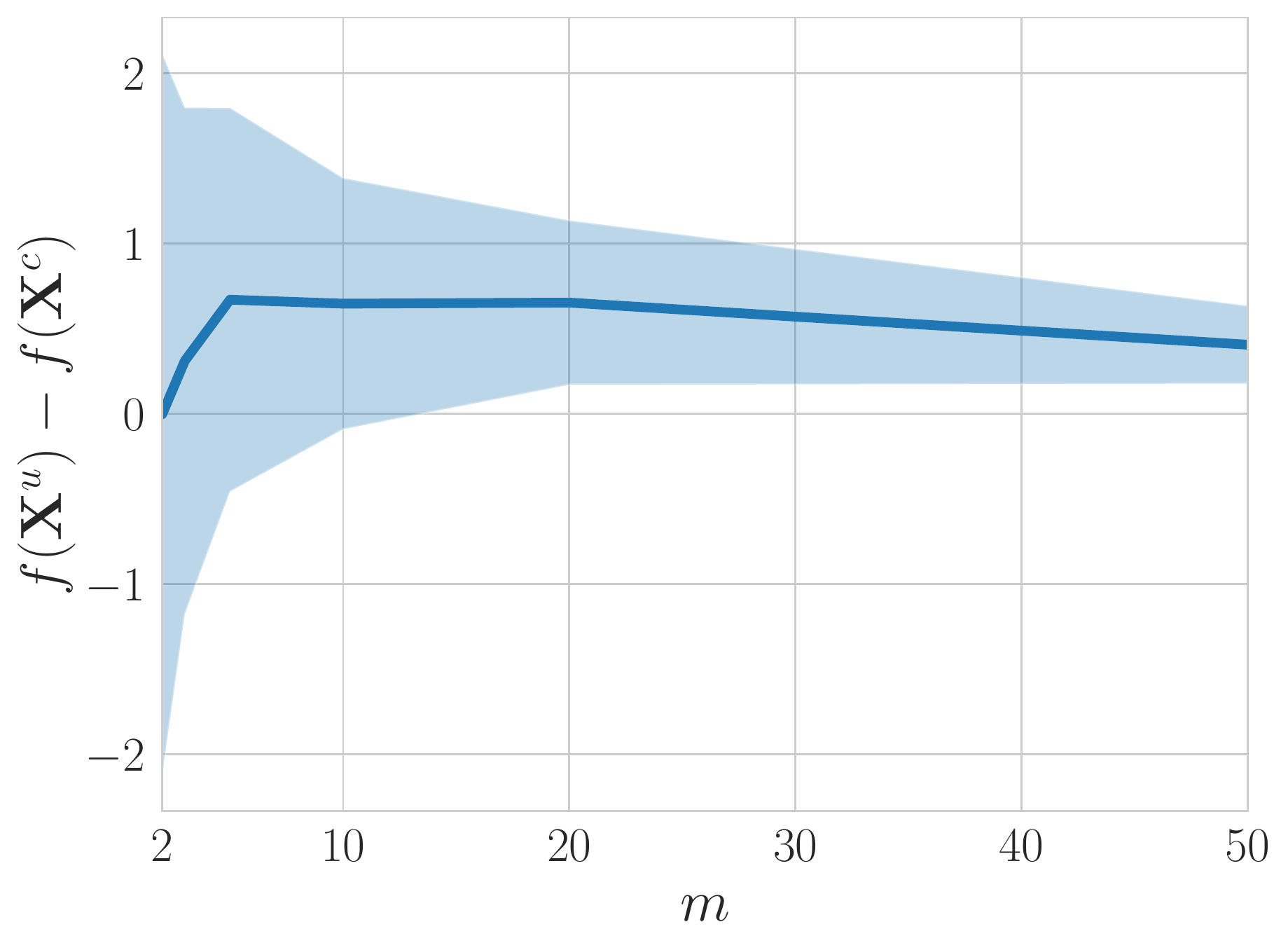}
    	\label{fig:effects_afo}
    }
	\caption{
    	Results on the effects of $L$ for set kernels and constrained optimization for acquisition functions.
    	The mean and standard deviation of each quantity are plotted, computed over 10 trials.
    	\emph{(left)}
    	The lines with \texttt{x} and \texttt{o} indicate kernel values and consumed times, respectively.
        The dashed line is the true set kernel value.
        \emph{(right)}
        $\bX^u$ and $\bX^c$ are the acquired sets obtained by the acquisition function optimization w/o and w/ constraints, respectively.\label{fig:effects}}
\end{figure}

\begin{table}[t]
	\caption{The effects of $L$ for set kernels.
	All settings follow the settings in~\figref{fig:effects_approx_L}.
	The numerical results are rounded to the three decimals, to show the effects precisely.
	\label{tab:kernel_n_d_m_l}}
	\vspace{10pt}
	\centering
	\begin{tabular}{cll}
		\toprule
		$L$ & \multicolumn{1}{c}{$\kset$} & \multicolumn{1}{c}{Time (sec.)} \\
		\midrule
		$1$ & $1.057 \times 10^{-5} \pm 1.387 \times 10^{-6}$ & $(2.776 \pm 0.877) \times 10^{-4}$ \\
		$2$ & $1.045 \times 10^{-5} \pm 4.512 \times 10^{-7}$ & $(3.131 \pm 0.118) \times 10^{-4}$ \\
		$5$ & $1.019 \times 10^{-5} \pm 1.168 \times 10^{-7}$ & $(9.283 \pm 0.089) \times 10^{-4}$ \\
		$10$ & $1.023 \times 10^{-5} \pm 9.980 \times 10^{-8}$ & $(3.133 \pm 0.822) \times 10^{-3}$ \\
		$20$ & $1.025 \times 10^{-5} \pm 9.305 \times 10^{-8}$ & $(9.150 \pm 3.169) \times 10^{-3}$ \\
		$50$ & $1.028 \times 10^{-5} \pm 4.955 \times 10^{-8}$ & $(5.487 \pm 1.075) \times 10^{-2}$ \\
		$100$ & $1.026 \times 10^{-5} \pm 3.030 \times 10^{-8}$ & $(2.282 \pm 0.134) \times 10^{-1}$ \\
		$200$ & $1.027 \times 10^{-5} \pm 1.961 \times 10^{-8}$ & $(9.253 \pm 0.490) \times 10^{-1}$ \\
		$500$ & $1.027 \times 10^{-5} \pm 1.372 \times 10^{-8}$ & $(5.887 \pm 0.159) \times 10^{0}$ \\
		$1000 (= m)$ & $1.027 \times 10^{-5}$ & $(2.295 \pm 0.027) \times 10^{1}$ \\
		\bottomrule
	\end{tabular}
\end{table}

We study the effect of $L$ for the set kernels.
Using a set generated from the standard normal distribution, which has 1,000 $50$-dimensional instances, 
we observe the effects of $L$ as shown in \figref{fig:effects_approx_L}.
$\kset$ converges to the true value as $L$ increases, 
and the variance of $\kset$ value is large when $L$ is small, 
as discussed in \secref{subsec:propertiesoftheapproximation}.
Moreover, the consumed time increases as $L$ increases.
We use Mat\'ern 5/2 kernel as a base kernel.
\tabref{tab:kernel_n_d_m_l} shows the effects of $L$ for set kernels.
As $L$ increases, $\kset$ value is converged to the true value and execution time increases.

\subsection{Constrained Acquisition Function Optimization\label{subsec:exp_const_opt}}

We demonstrate the effects of the constrained acquisition function optimization,
compared to the vanilla optimization method that concatenates a set to a single vector.
In this paper, we use CMA-ES~\cite{HansenN2016arxiv} as an acquisition function optimization method, 
which is widely used in Bayesian optimization~\cite{BergstraJ2011neurips,WangZ2014aistats}.
As we mentioned in \secref{subsec:afo_sets},
the constrained CMA-ES is more sample-efficient than the vanilla CMA-ES.
\figref{fig:effects_afo} (i.e., a minimization problem) represents 
that the function values determined by the constrained method 
are always smaller than the values by the unconstrained method,
because we fix the number of initial samples (i.e., 5).
Moreover, the variance of them decreases, as $m$ is large.
For this experiment, we measure the acquisition performance 
where 20 fixed historical observations are given,
and Synthetic 1 described in \secref{subsec:exp_syn} is used.
Note that the kernel approximation is not applied.

\begin{figure}[t]
	\centering
	\subfigure[Vector]
	{
		\includegraphics[width=0.30\textwidth]{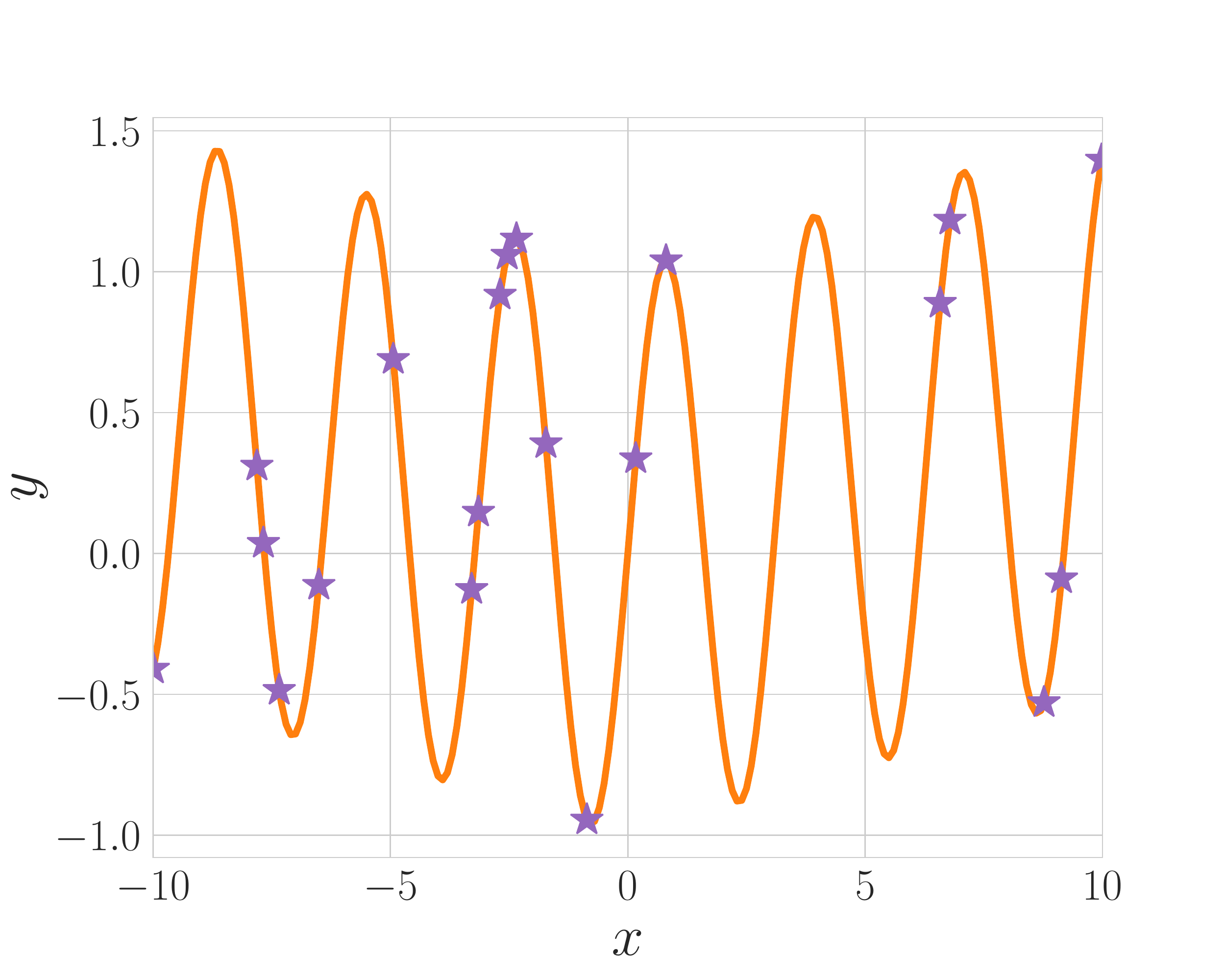}
		\label{subfig:syn_1_concat}
	}
	\subfigure[Split]
	{
		\includegraphics[width=0.30\textwidth]{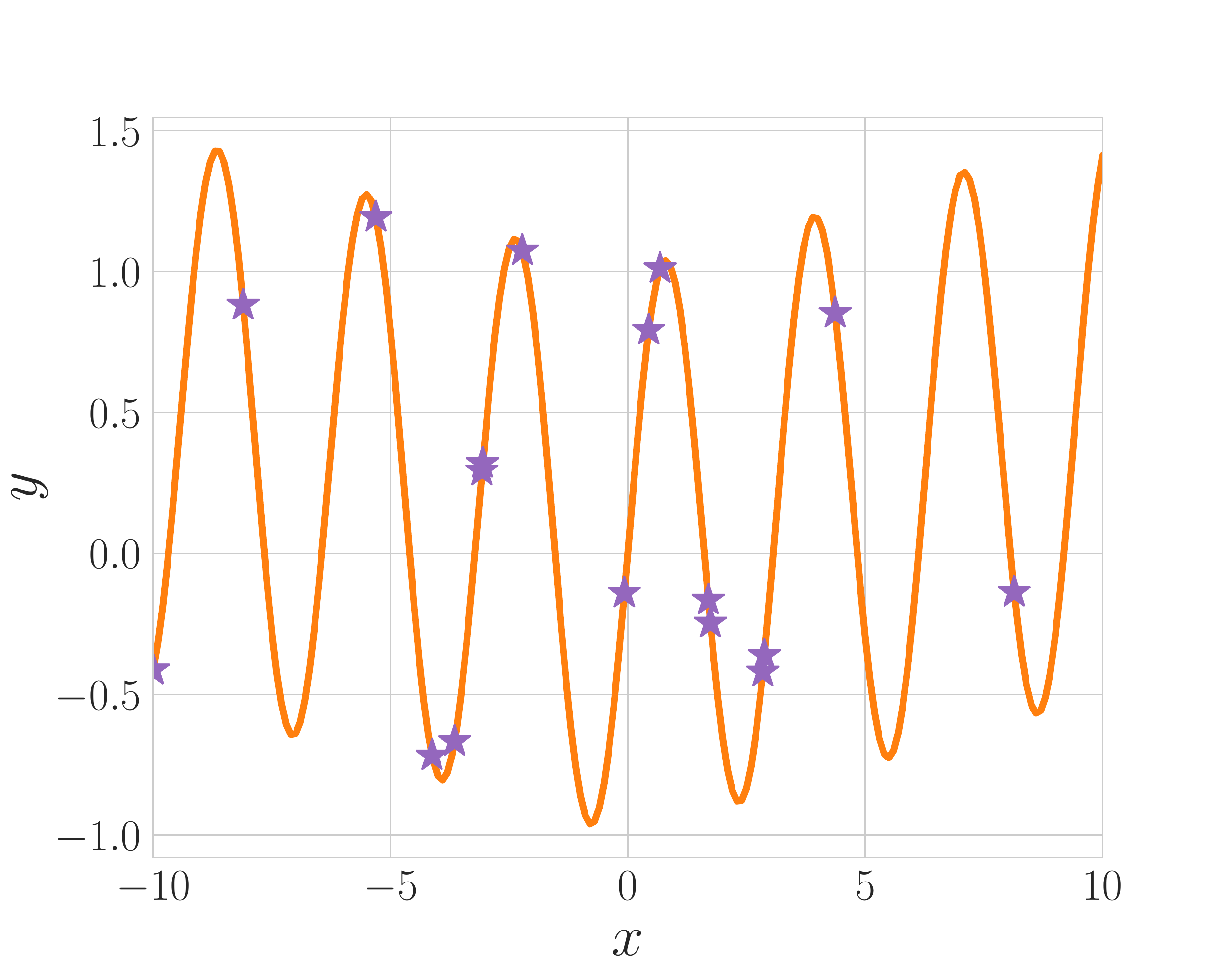}
		\label{subfig:syn_1_split}
	}
	\subfigure[Ours (w/o approx.)]
	{
		\includegraphics[width=0.30\textwidth]{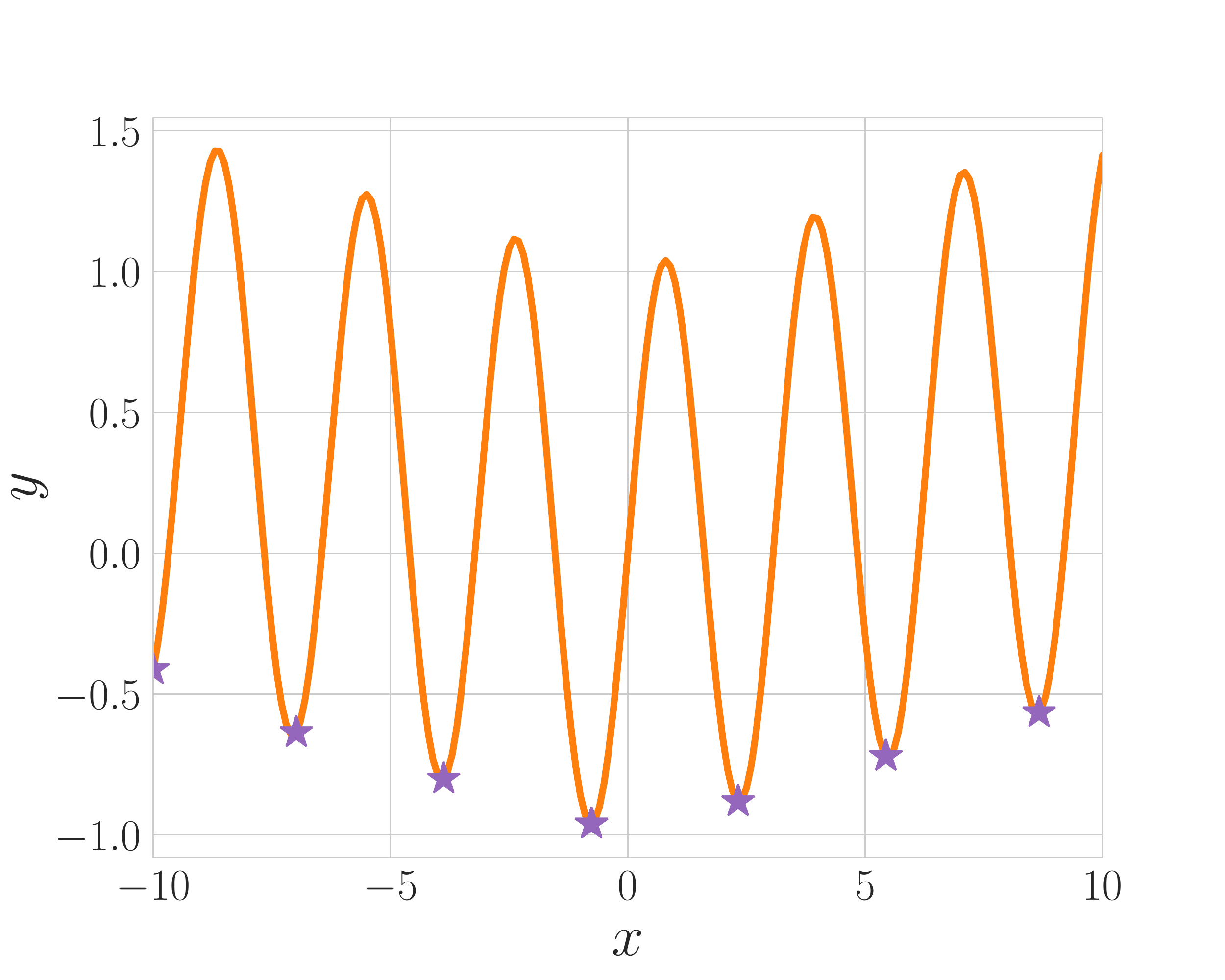}
		\label{subfig:syn_1_ours}
	}
	\subfigure[Vector]
	{
		\includegraphics[width=0.30\textwidth]{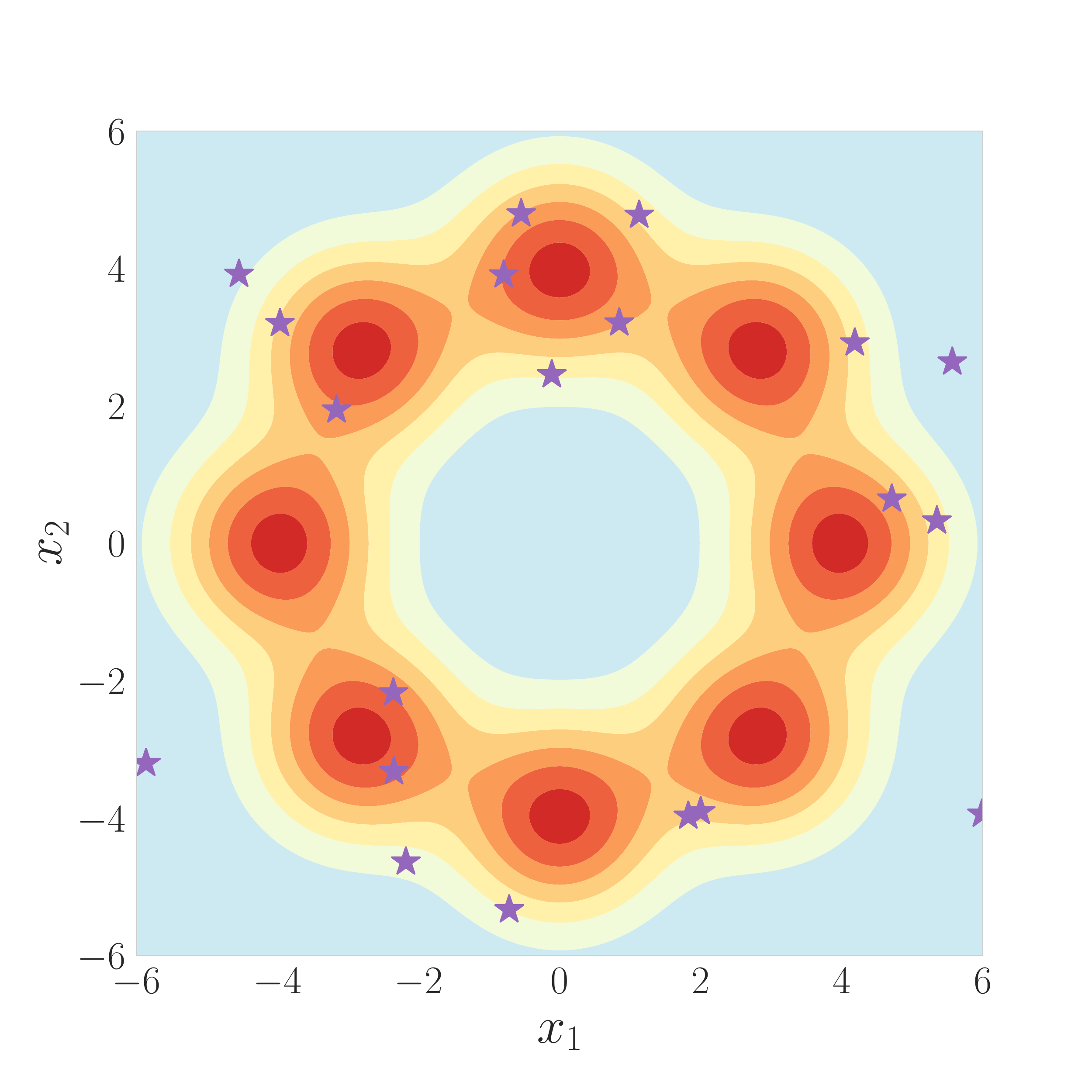}
		\label{subfig:syn_2_concat}
	}
	\subfigure[Split]
	{
		\includegraphics[width=0.30\textwidth]{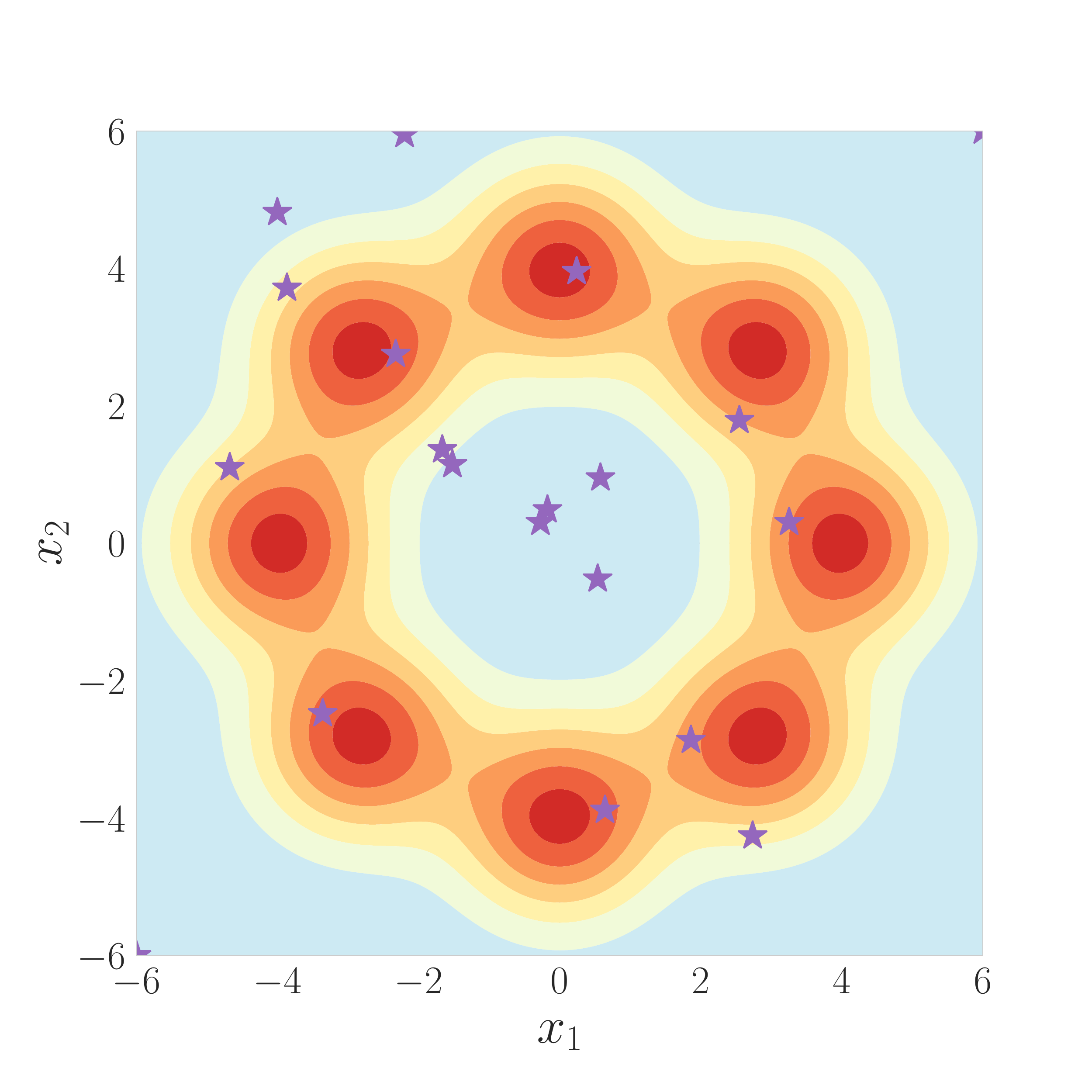}
		\label{subfig:syn_2_split}
	}
	\subfigure[Ours (w/o approx.)]
	{
		\includegraphics[width=0.30\textwidth]{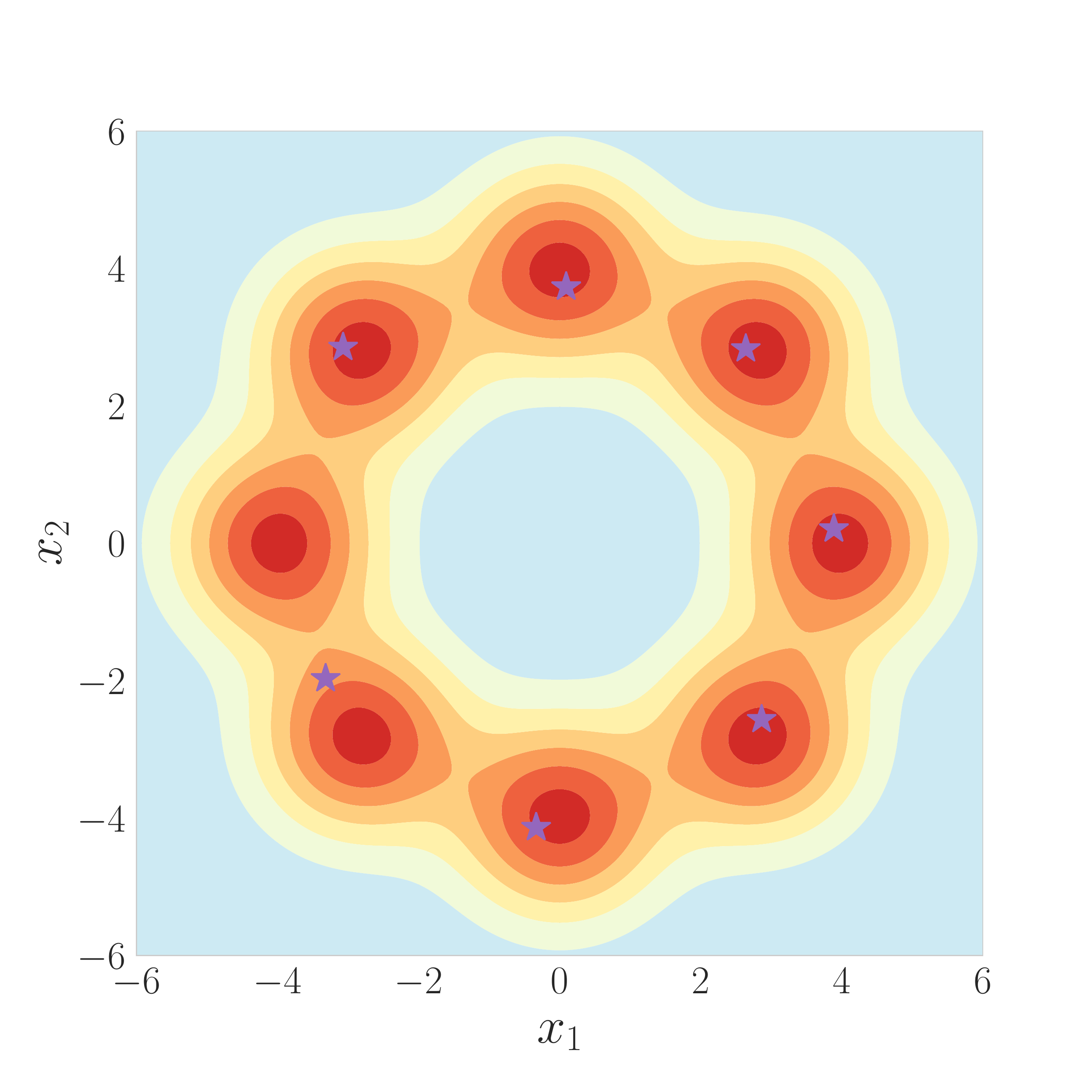}
		\label{subfig:syn_2_ours}
	}
	\caption{Examples of one of the best acquisition results 
	(i.e., purple stars indicate instances in the acquired set) via Vector, Split, and Ours (w/o approximation).
	For Synthetic 1 (first row) and Synthetic 2 (second row), $m$ is set to $20$.}
	\label{fig:syn_1_2}
\end{figure}

\begin{figure}[t]
	\centering
	\subfigure[Synthetic 1]
	{
		\includegraphics[width=0.37\textwidth]{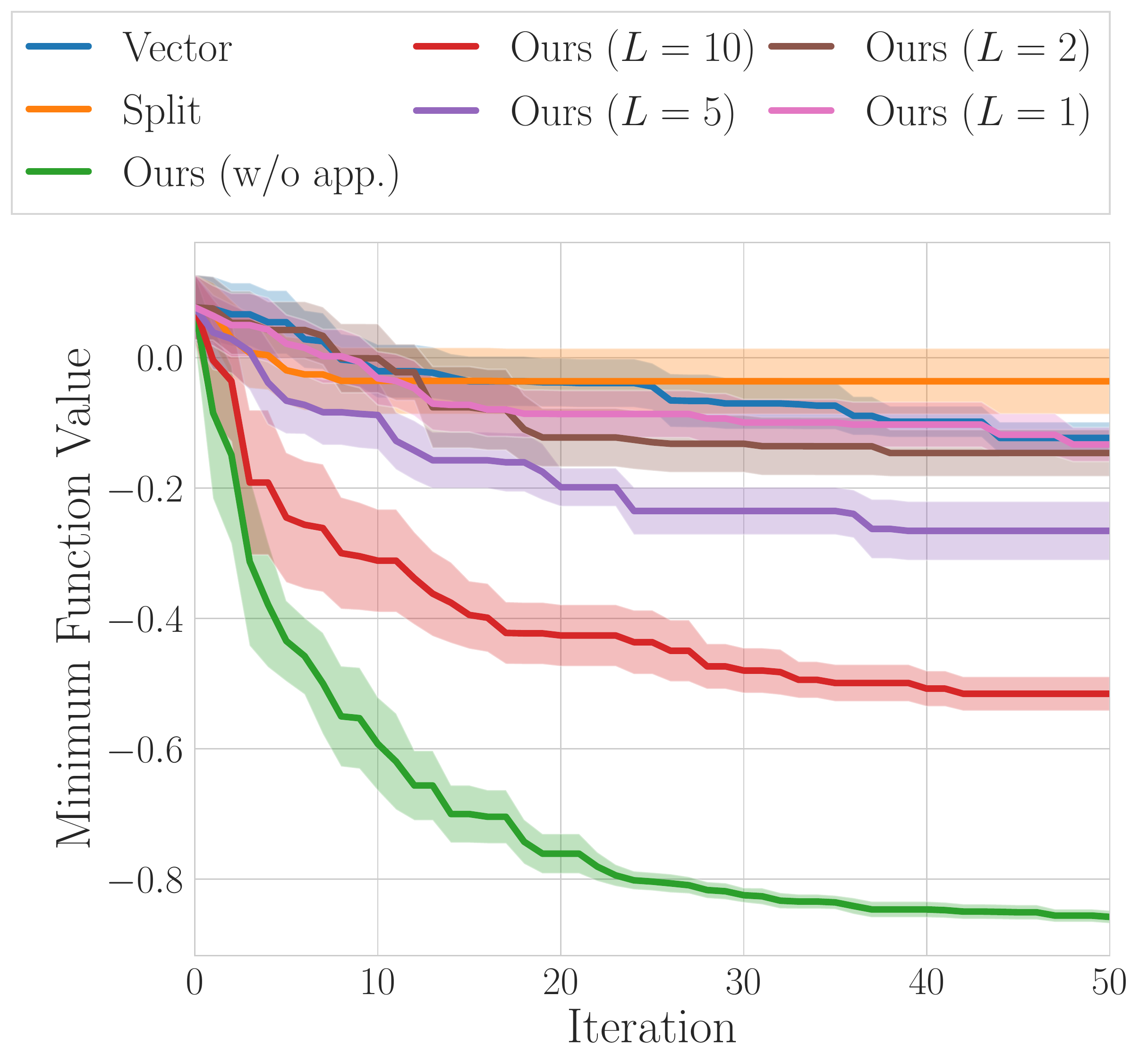}
		\label{subfig:syn_1_20_graph}
	}
	\quad\quad\quad
	\subfigure[Synthetic 2]
	{
		\includegraphics[width=0.37\textwidth]{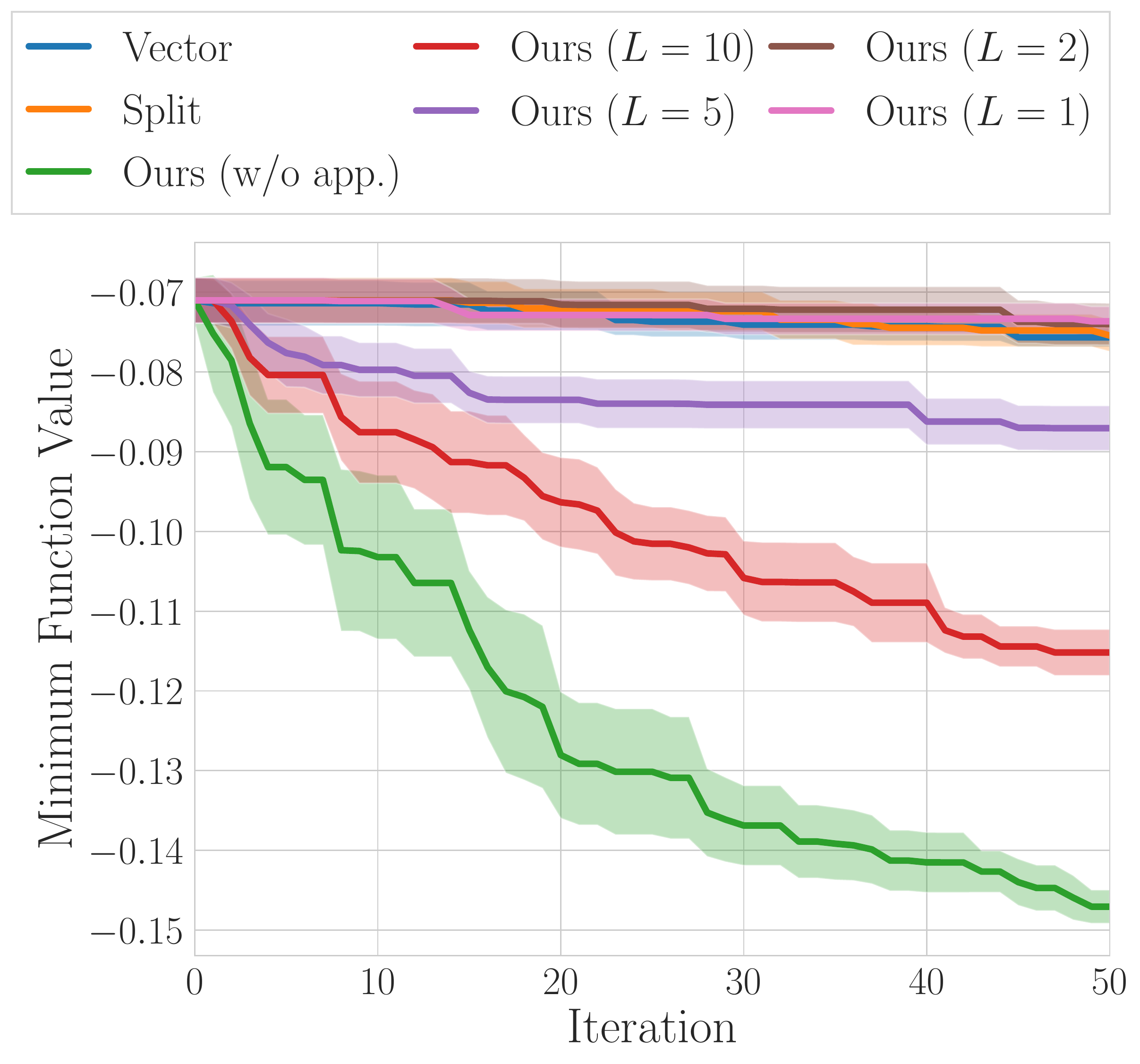}
		\label{subfig:syn_2_20_graph}
	}
	\caption{Results on optimizing two synthetic functions.
	As presented in \figref{fig:syn_1_2}, $m$ is set to $20$.
	All experiments are repeated 10 times.}
	\label{fig:results_syn_1_2}
\end{figure}

\subsection{Synthetic Functions\label{subsec:exp_syn}}

We test two synthetic functions to show Bayesian optimization over sets is a valid approach 
to find an optimal set that minimizes an objective function
$f:\calXset \to \bbR$.  In each setting, there is an auxiliary
function $g: \calX \to \bbR$, and $f$ is defined as
$f(\bX) = \frac{1}{m} \sum_{i = 1}^m g(\bx_i)$.
The $g$ functions are designed to be multi-modal, giving
the opportunity for the set $\bX$ to contain $\bx_i$ values
from each of the modes in the domain.  Additionally, as is
expected, $f$ is permutation invariant (any ordering imposed
on the elements of $\bx$ is immaterial).

\noindent \textbf{Synthetic 1}
We consider $d=1$, $m=20$ and choose $g$ to be a simple periodic function:
\begin{equation}
    g(\bx) = \sin(2 \|\bx\|_2) + | 0.05 \|\bx\|_2 |.
\end{equation}

\noindent \textbf{Synthetic 2}
We consider $d=2$, $m=20$ and a $g$ function which is
the sum of probability density functions:
\begin{equation}
    g(\bx) = - \sum_{i = 1}^8 p(\bx; \mu_i, \Sigma_i),
\end{equation}
where $p$ is the normal density function
with $\mu_i$ depicted in \figref{fig:syn_1_2} and $\Sigma_i = \sI_2$.

As shown in \figref{fig:syn_1_2}, both of these functions
have a clear multimodal structure, allowing for optimal sets
to contain points which are clustered in a single local minima
or to be spread out through the domain in several local minima.
\figref{fig:results_syn_1_2} shows that Vector and Split strategies have difficulty
optimizing the functions.
On the other hand, our proposed method finds optimal
outcomes more effectively.\footnote{
While not our concern here, it is possible that
some amount of distance between points in the set $\bX$
would be desired.  If that were the case, such a desire
could be enforced in the function $f$.}
We study the impact of $L$ when optimizing these two synthetic functions; 
a smaller $L$ should yield faster computations, 
but also a worse approximation
$\sKtilde$ to the true $\sK$ matrix (when $L = m$).

\tabref{tab:effects_syn_20} represents a convergence quality and its execution time for the synthetic functions defined in this work.
As expected, the execution time decreases as $L$ decreases.

\subsection{Clustering Algorithm Initialization\label{subsec:exp_clustering}}

\begin{table}[t]
	\caption{
	Convergence quality and its execution time on two synthetic functions where $m=20$.
	All settings follow the settings in \figref{fig:syn_1_2} and \figref{fig:results_syn_1_2}.
	\label{tab:effects_syn_20}}
	\vspace{10pt}
	\centering
	\begin{tabular}{ccrcr}
		\toprule
		& \multicolumn{2}{c}{\textbf{Synthetic 1}} & \multicolumn{2}{c}{\textbf{Synthetic 2}} \\
		$L$ & Minimum & \multicolumn{1}{c}{Time ($10^3$ sec.)} & Minimum & \multicolumn{1}{c}{Time ($10^3$ sec.)} \\
		\midrule
		$1$ & $-0.133 \pm 0.052$ & $5.682 \pm 0.279$ & $-0.074 \pm 0.004$ & $6.468 \pm 0.130$ \\
		$2$ & $-0.146 \pm 0.071$ & $8.003 \pm 0.156$ & $-0.074 \pm 0.005$ & $9.768 \pm 0.136$ \\
		$5$ & $-0.266 \pm 0.090$ & $16.299 \pm 0.253$ & $-0.087 \pm 0.006$ & $20.451 \pm 0.033$ \\
		$10$ & $-0.516 \pm 0.052$ & $31.010 \pm 0.810$ & $-0.115 \pm 0.006$ & $40.099 \pm 0.085$ \\
		$20 (= m)$ & $-0.858 \pm 0.019$ & $91.589 \pm 2.339$ & $-0.147 \pm 0.004$ & $85.266 \pm 0.297$ \\
		\bottomrule
	\end{tabular}
\end{table}

We initialize clustering algorithms for dataset 
$\calP = [\bp_1, \ldots, \bp_N]$ with Bayesian optimization over sets.
For these experiments, we add four additional baselines for
clustering algorithms:

\noindent \textbf{Random}
This baseline randomly draws $k$ points from 
a compact space $\subset \bbR^d$.

\noindent \textbf{Data}
This baseline randomly samples $k$ points from a dataset $\calP$.
It is widely used in initializing a clustering algorithm.

\noindent \textbf{($k$-means only) $k$-means++~\cite{ArthurD2007soda}}
This is a method for $k$-means clustering with the intuition 
that spreading out initial cluster centers is better than the Data baseline.

\noindent \textbf{(GMM only) $k$-means}
This baseline sets initial cluster centers as the results of $k$-means clustering.

To fairly compare the baselines to our methods, 
the baselines are trained by the whole datasets without splitting.
To compare with the baselines fairly, 
Random, Data, $k$-means++~\cite{ArthurD2007soda}, $k$-means are run 1,000 times.
In Bayesian optimization settings, we split a dataset to training (70\%) and test (30\%) datasets.
After finding the converged cluster centers $\{ \bc_1, \ldots, \bc_k \}$ with training dataset,
the adjusted Rand index (ARI) is computed by test dataset.
The algorithms are optimized over $1 - \textrm{ARI}$.
All clustering models are implemented using \texttt{scikit-learn}~\cite{PedregosaF2011jmlr}.

We test two clustering algorithms for synthetic datasets: 
(i) $k$-means clustering and 
(ii) Gaussian mixture model (GMM).
In addition, 
two real-world datasets are tested to initialize $k$-means clustering:
(i) Handwritten Digits dataset~\cite{DuaD2019uci} and (ii) NIPS Conference Papers dataset~\cite{PerroneV2017jmlr}.
As shown in \figref{fig:kmeans_mog} 
and \figref{fig:kmeans_real}, 
our methods outperform other application-agnostic baselines 
as well as four baselines for clustering methods.

\begin{figure}[t]
	\centering
    	\subfigure[$k$-means clustering]
    	{
    		\includegraphics[width=0.37\textwidth]{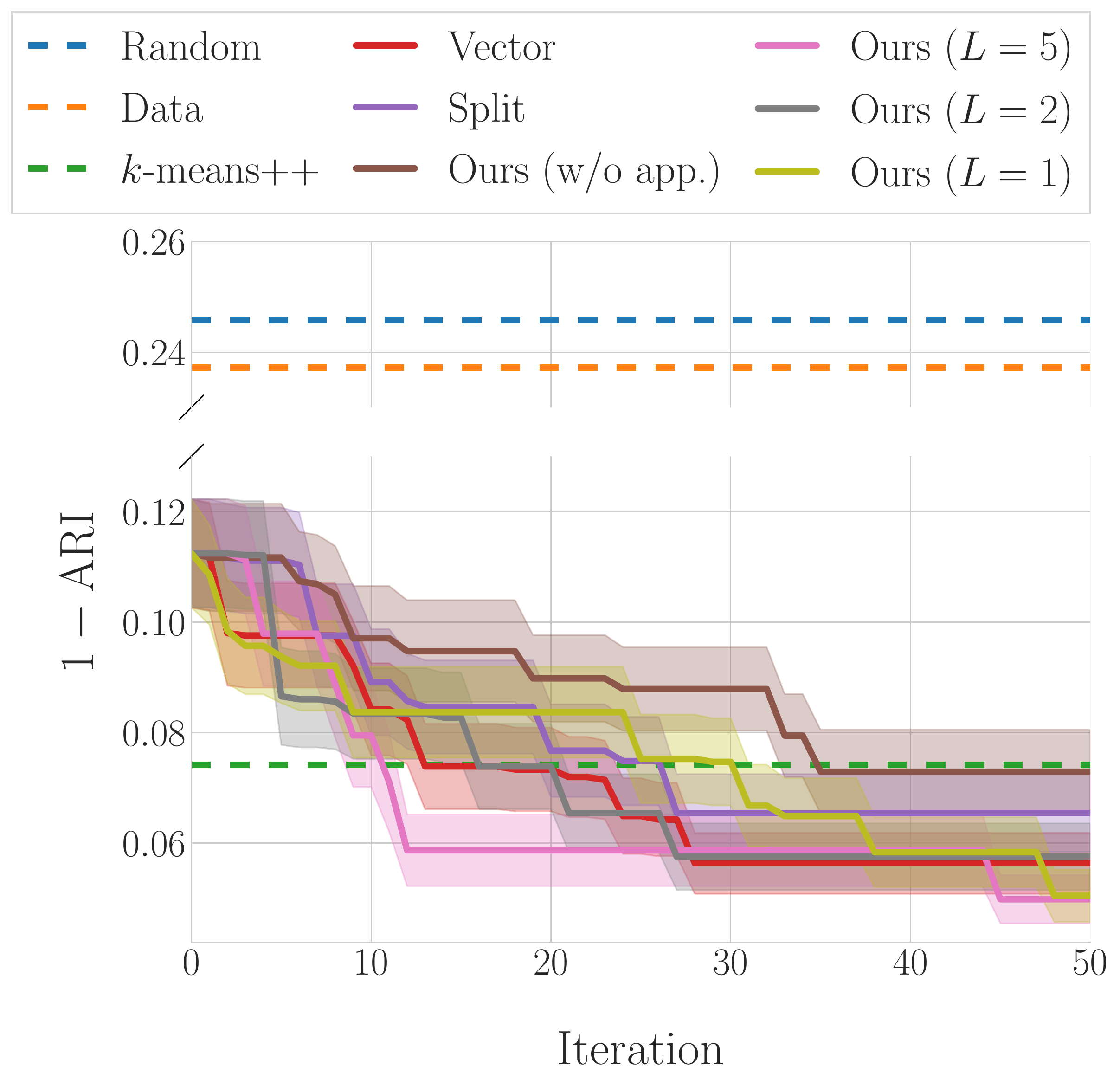}
    		\label{subfig:kmeans_10}
    	}
	\quad\quad\quad
    	\subfigure[Gaussian mixture model]
    	{
    		\includegraphics[width=0.37\textwidth]{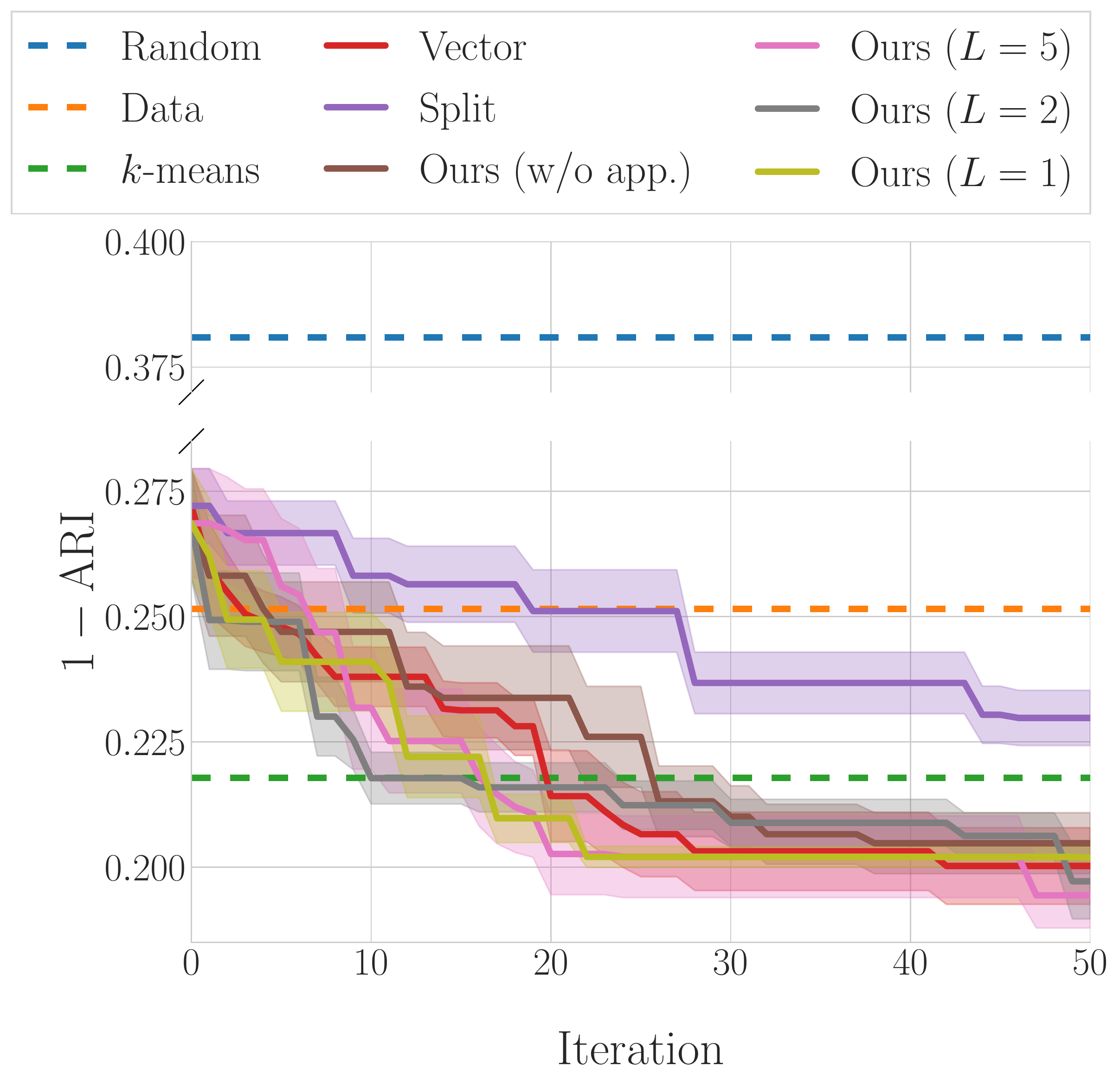}
    		\label{subfig:mog_10}
    	}
	\caption{Results on initializing clustering algorithms: $k$-means clustering and Gaussian mixture model for synthetic datasets.\label{fig:kmeans_mog}}
\end{figure}

\begin{figure}[t]
	\centering
    	\subfigure[Handwritten Digits]
    	{
    		\includegraphics[width=0.37\textwidth]{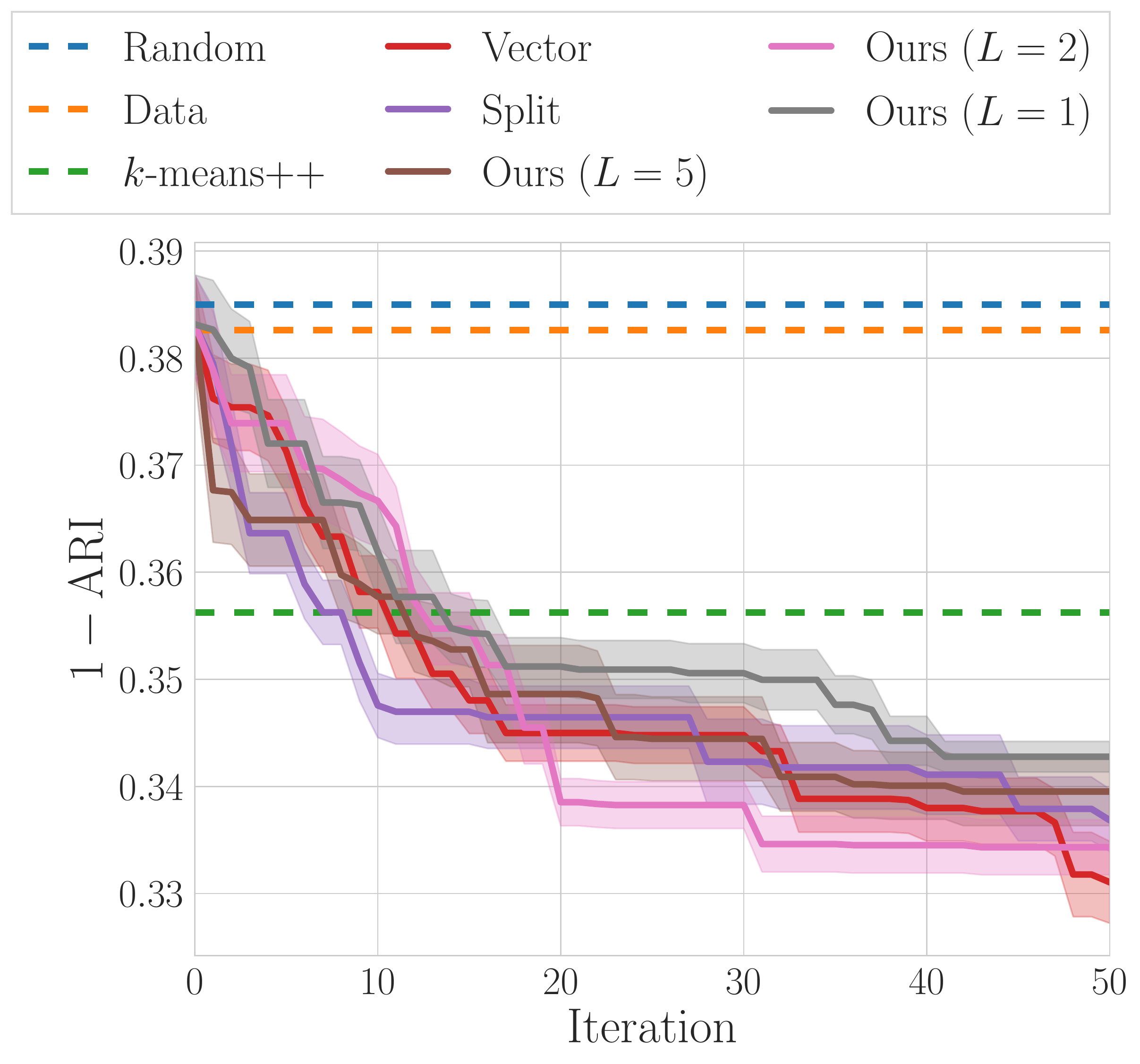}
    		\label{subfig:real_1}
    	}
	\quad\quad\quad
    	\subfigure[NIPS Conference Papers]
    	{
    		\includegraphics[width=0.37\textwidth]{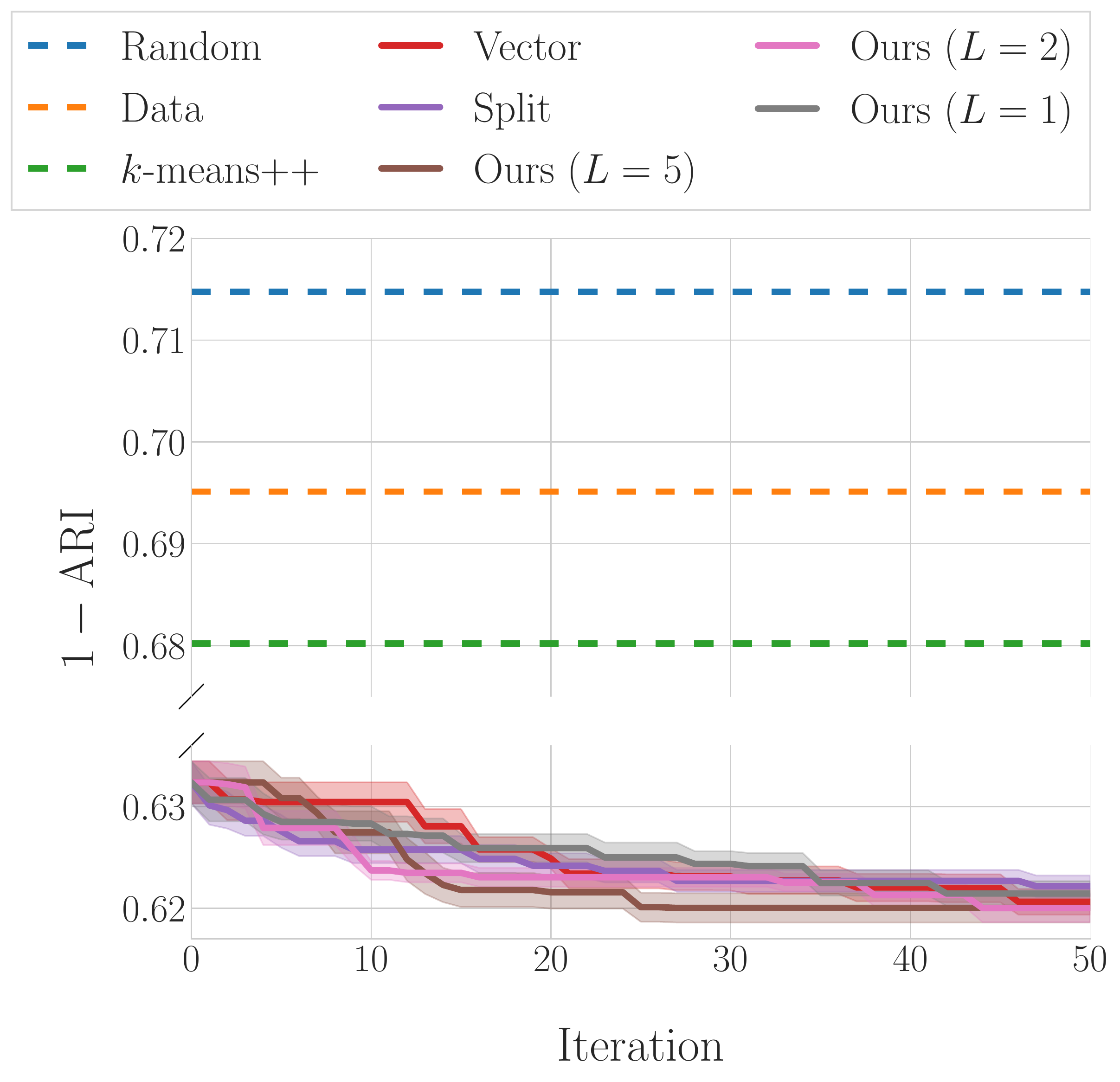}
    		\label{subfig:real_2}
    	}
	\caption{Results on initializing $k$-means clustering for Handwritten Digits and NIPS Conference Papers datasets.\label{fig:kmeans_real}}
\end{figure}

\begin{table}[t]
	\caption{Convergence quality and its execution time on $k$-means clustering and Gaussian mixture model.
	All settings follow the settings in \figref{fig:kmeans_mog}.
	\label{tab:effects_clustering_syn}}
	\vspace{10pt}
	\centering
	\begin{tabular}{ccrcr}
		\toprule
		& \multicolumn{2}{c}{$\boldsymbol k$-\textbf{means clustering}} & \multicolumn{2}{c}{\textbf{Gaussian mixture model}} \\
		$L$ & $1 - $ARI & \multicolumn{1}{c}{Time ($10^3$ sec.)} & $1 - $ARI & \multicolumn{1}{c}{Time ($10^3$ sec.)} \\
		\midrule
		$1$ & $0.051 \pm 0.024$ & $3.140 \pm 0.051$ & $0.202 \pm 0.011$ & $2.364 \pm 0.178$\\
		$2$ & $0.058 \pm 0.030$ & $5.034 \pm 0.145$ & $0.197 \pm 0.038$ & $3.873 \pm 0.148$\\
		$5$ & $0.050 \pm 0.022$ & $13.891 \pm 0.327$ & $0.194 \pm 0.033$ & $13.105 \pm 0.675$\\
		$10 (= m)$ & $0.073 \pm 0.038$ & $45.527 \pm 2.475$ & $0.205 \pm 0.030$ & $49.134 \pm 1.469$ \\
		\bottomrule
	\end{tabular}
\end{table}

\begin{table}[t]
	\caption{Convergence quality and its execution time on $k$-means clustering for Handwritten Digits and NIPS Conference Papers datasets.
	All settings follow the settings in \figref{fig:kmeans_real}.
	\label{tab:effects_clustering_real}}
	\vspace{10pt}
	\centering
	\begin{tabular}{ccrcr}
		\toprule
		& \multicolumn{2}{c}{\textbf{Handwritten Digits}} & \multicolumn{2}{c}{\textbf{NIPS Conference Papers}} \\
		$L$ & $1 - $ARI & \multicolumn{1}{c}{Time ($10^3$ sec.)} & $1 - $ARI & \multicolumn{1}{c}{Time ($10^3$ sec.)} \\
		\midrule
		$1$ & $0.343 \pm 0.007$ & $39.347 \pm \, \ 6.533$ & $0.621 \pm 0.006$ & $4.757 \pm 0.138$\\
		$2$ & $0.334 \pm 0.013$ & $47.550 \pm \, \ 1.832$ & $0.620 \pm 0.007$ & $9.410 \pm 0.260$\\
		$5$ & $0.340 \pm 0.016$ & $177.170 \pm 43.063$ & $0.620 \pm 0.007$ & $38.521 \pm 1.551$\\
		\bottomrule
	\end{tabular}
\end{table}

\noindent \textbf{Synthetic Datasets}
We generate a dataset sampled from Gaussian distributions, where $N = 500$, $d = 5$, and $k = 10$.

\noindent \textbf{Real-World Datasets}
Two real-world datasets are tested:
(i) Handwritten Digits dataset~\cite{DuaD2019uci} and (ii) NIPS Conference Papers dataset~\cite{PerroneV2017jmlr}.
Handwritten Digits dataset contains 0--9 digit images that can be expressed as $N=1797$, $d=64$, and $k=10$.
NIPS Conference Papers dataset is composed of the papers published from 1987 to 2015.
The features of each example are word frequencies,
and this dataset can be expressed as $N=5811$, $d=11463$, and $k=20$.
However, without any techniques for reducing the dimensionality,
this dataset is hard to apply the clustering algorithm.
We choose 200 dimensions in random when creating the dataset for these experiments, 
because producing the exact clusters for entire dimensions is not our interest in this paper.

Because the real-world datasets for clustering are difficult to specify truths,
we determine truths as class labels for Handwritten Digits dataset~\cite{DuaD2019uci} 
and clustering results via Ward hierarchical clustering~\cite{WardJH63jasa} for NIPS Conference Papers dataset~\cite{PerroneV2017jmlr}.

The function of interest in the $k$-means clustering setting
is the \emph{converged} clustering residual
\begin{equation}
    k\textrm{-means}(\{\bx_1, \ldots, \bx_k\}) = \sum_{i = 1}^N \sum_{j = 1}^k
        w_{ij}\|\bp_i - \bc_j\|^2_2,
    \label{eqn:kmeans_obj}
\end{equation}
where
$\{\bx_1, \ldots, \bx_k\}$ is the set of proposed initial cluster centers,
$\{\bc_1, \ldots, \bc_k\}$ is the set of
converged cluster centers~\cite{LloydS1982ieeetit},
and $w_{ij}$ are softmax values from the pairwise distances.
Here, the fact that $\bc_j$ is a function of $\bX$ and $\calP$ 
is omitted for notational simplicity.
The set of converged cluster centers is determined 
through an iterative strategy which is highly dependent
on the initial points $\bX$ to converge to effective centers.

In contrast to $k$-means clustering,
the GMM estimates parameters of Gaussian distributions
and mixing parameters between the distributions.
Because it is difficult to minimize negative
log-likelihood of the observed data, 
we fit the GMM using
expectation-maximization algorithm~\cite{Dempster77jrssB}.
Similarly to $k$-means clustering, this requires
initial guesses $\bX$ to converge to cluster centers
$\{\mathbf{c}_1, \ldots, \mathbf{c}_k\}$.

\tabref{tab:effects_clustering_syn} shows convergence qualities and their execution time on $k$-means clustering algorithm and GMM
for synthetic datasets,
and \tabref{tab:effects_clustering_real} represents the qualities and their execution for Handwritten Digits and NIPS Conference Papers datasets.
Similar to \tabref{tab:kernel_n_d_m_l}, 
the computational cost increases as $L$ increases.

\begin{figure}[t]
	\centering
	\subfigure[100 examples]
	{
		\includegraphics[width=0.37\textwidth]{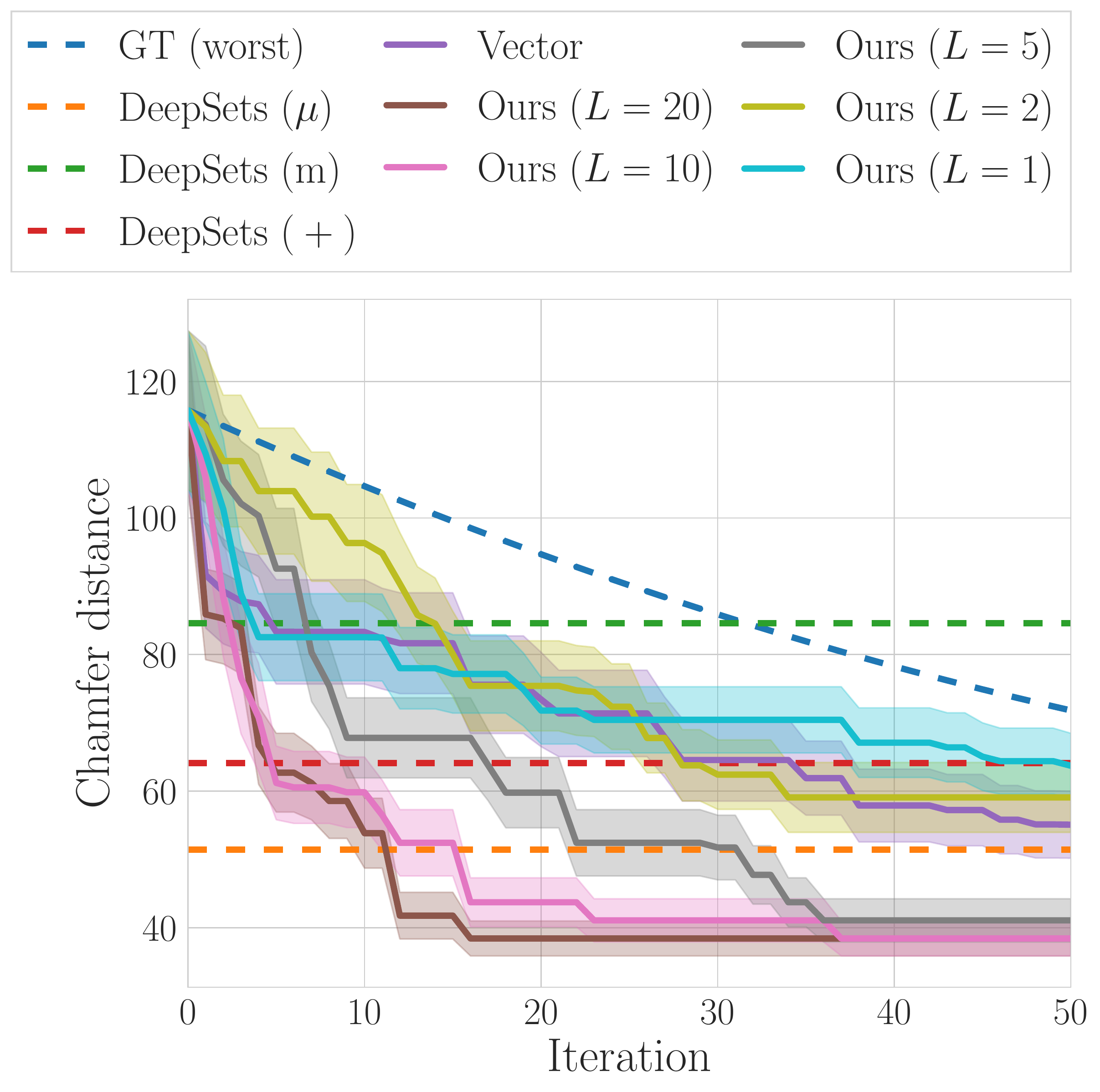}
		\label{fig:pc_100}
	}
	\quad\quad\quad
	\subfigure[Entire examples]
	{
		\includegraphics[width=0.37\textwidth]{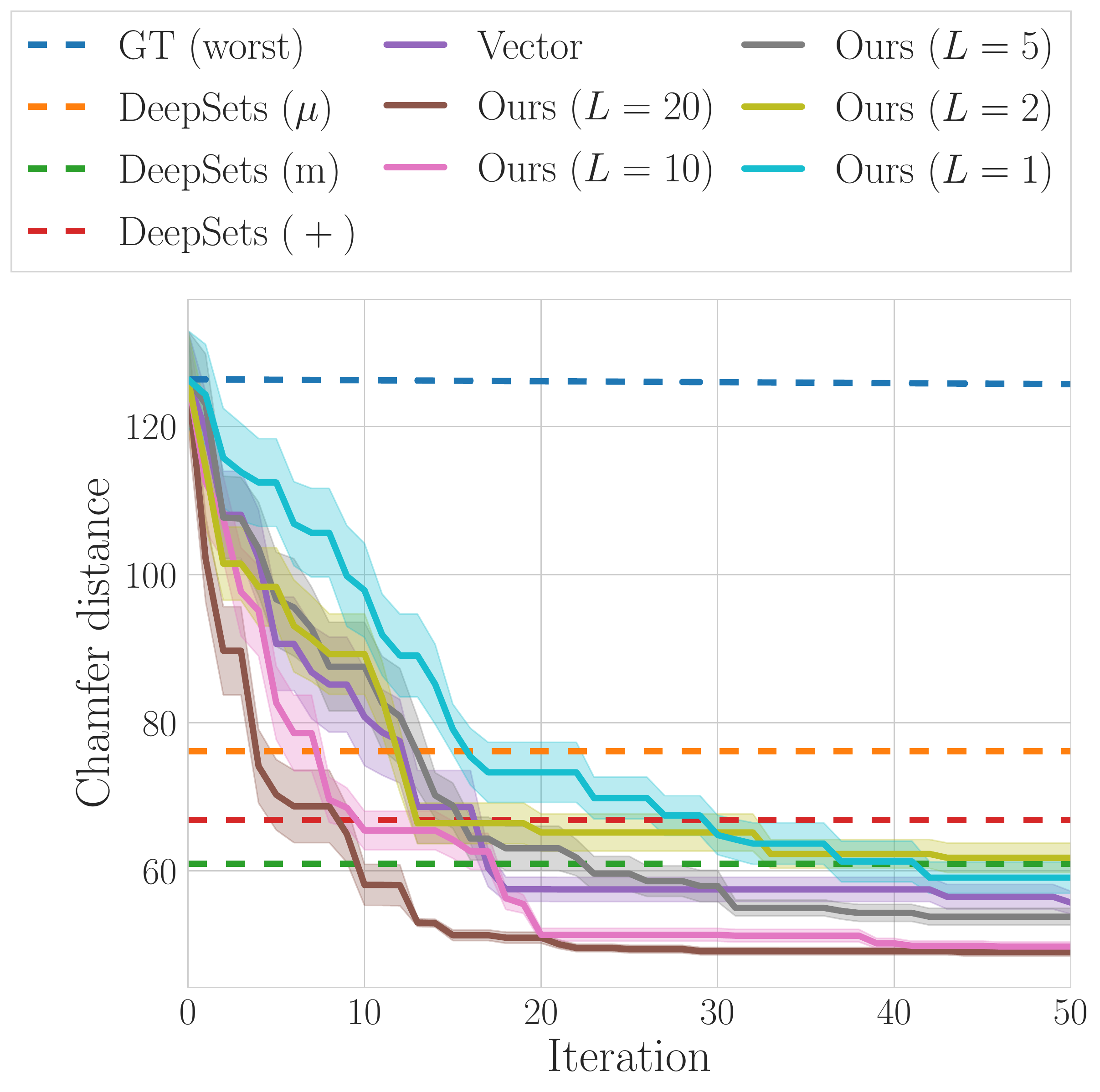}
		\label{fig:pc_12311}
	}
	\\
	\subfigure[Query 1]
	{
		\includegraphics[width=0.14\textwidth]{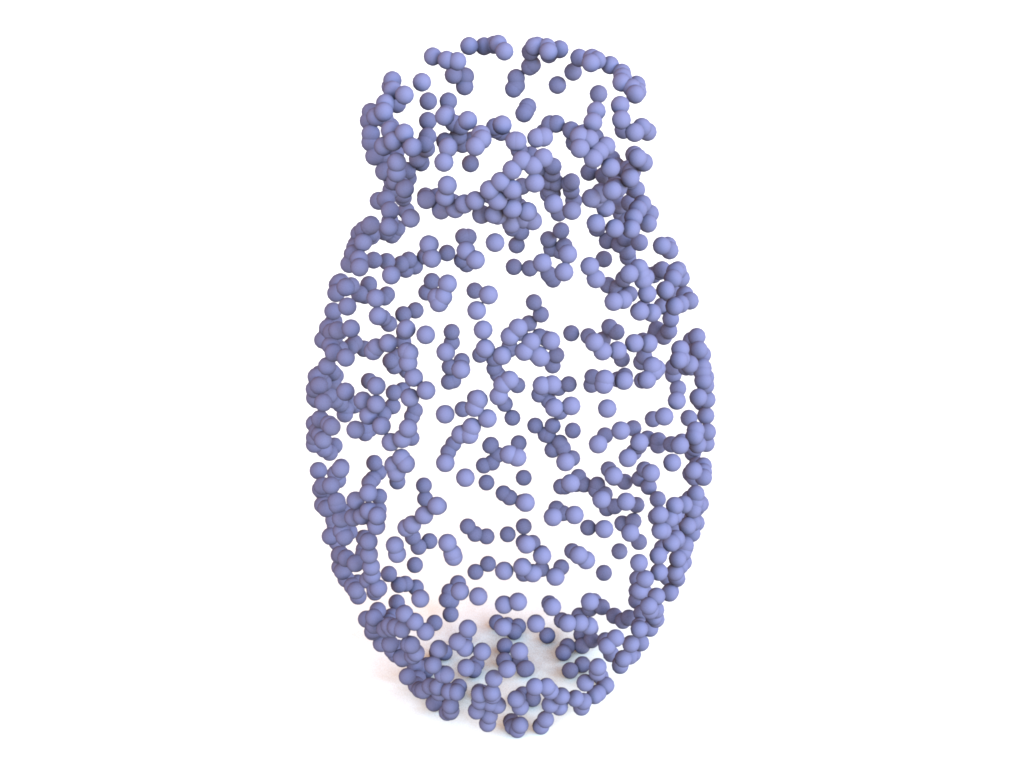}
		\label{fig:pc_target_1}
	}
	\subfigure[NN 1]
	{
		\includegraphics[width=0.14\textwidth]{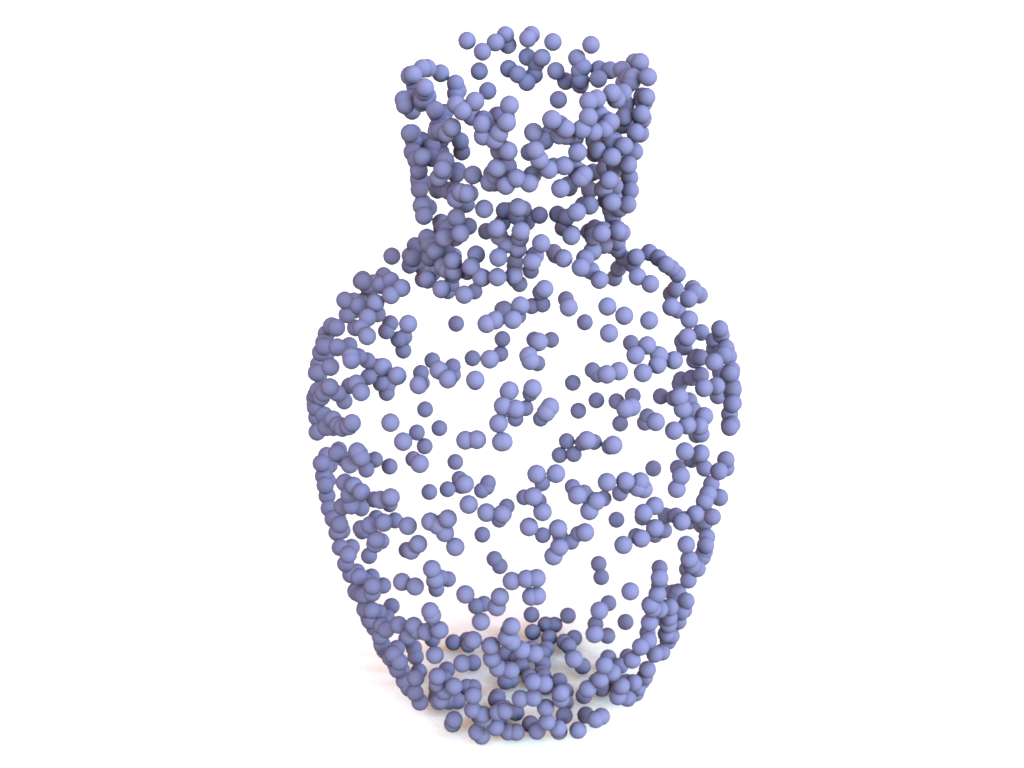}
		\label{fig:pc_nn_1}
	}
	\subfigure[Query 2]
	{
		\includegraphics[width=0.14\textwidth]{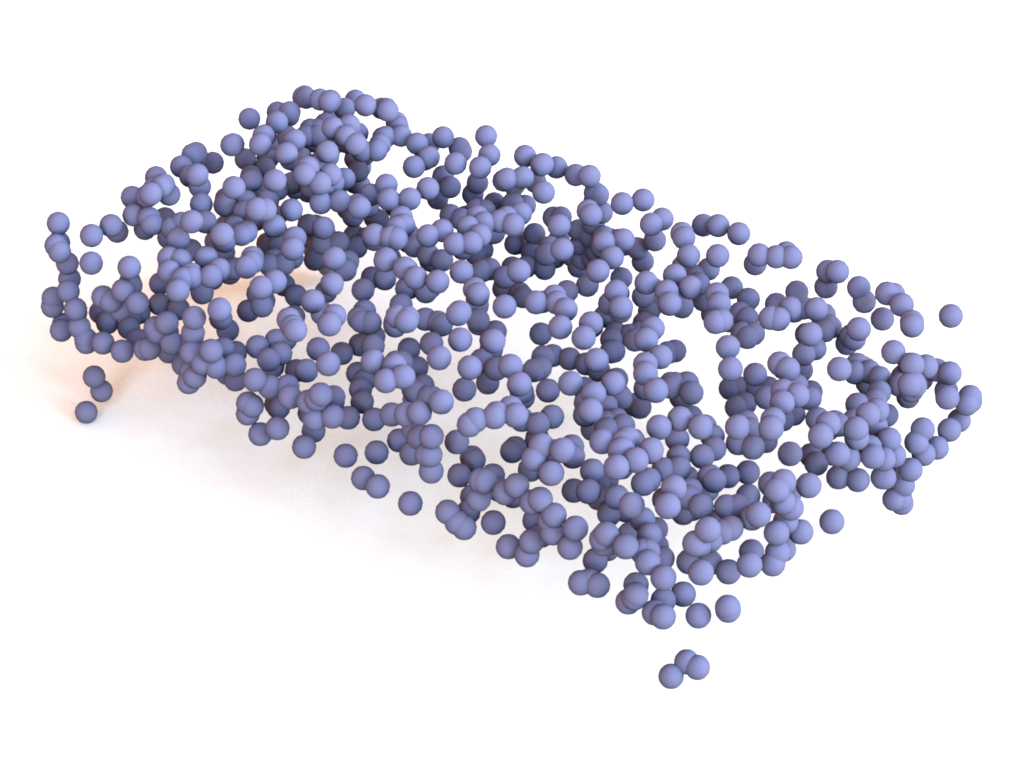}
		\label{fig:pc_target_2}
	}
	\subfigure[NN 2]
	{
		\includegraphics[width=0.14\textwidth]{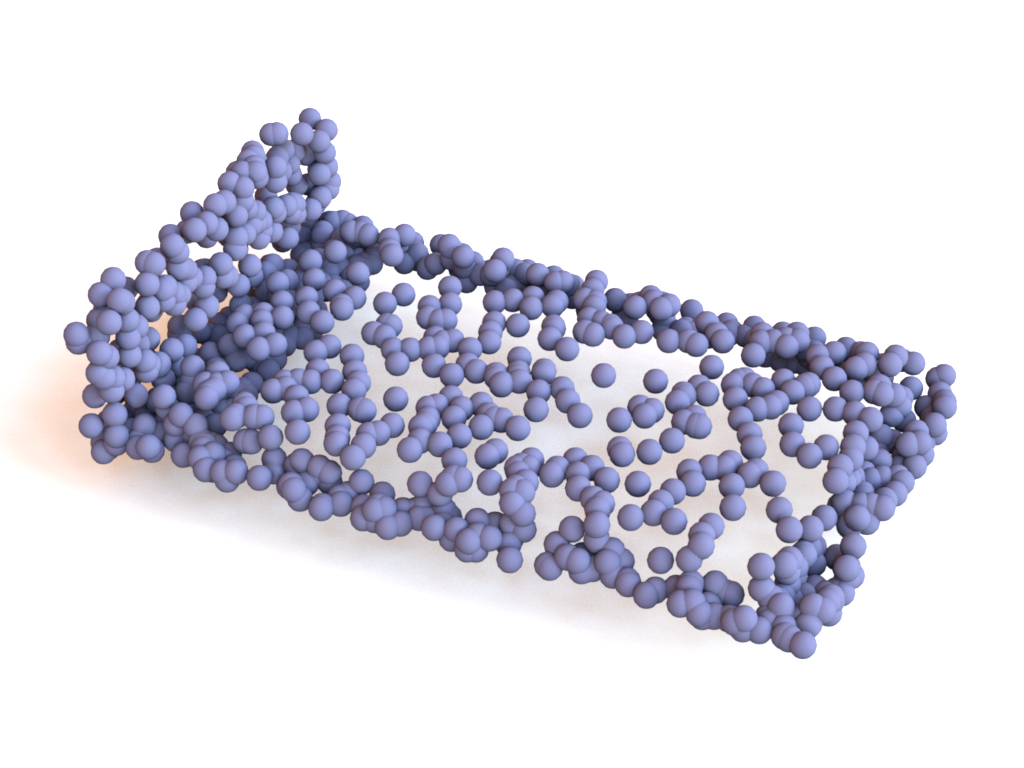}
		\label{fig:pc_best_2}
	}
	\subfigure[Query 3]
	{
		\includegraphics[width=0.14\textwidth]{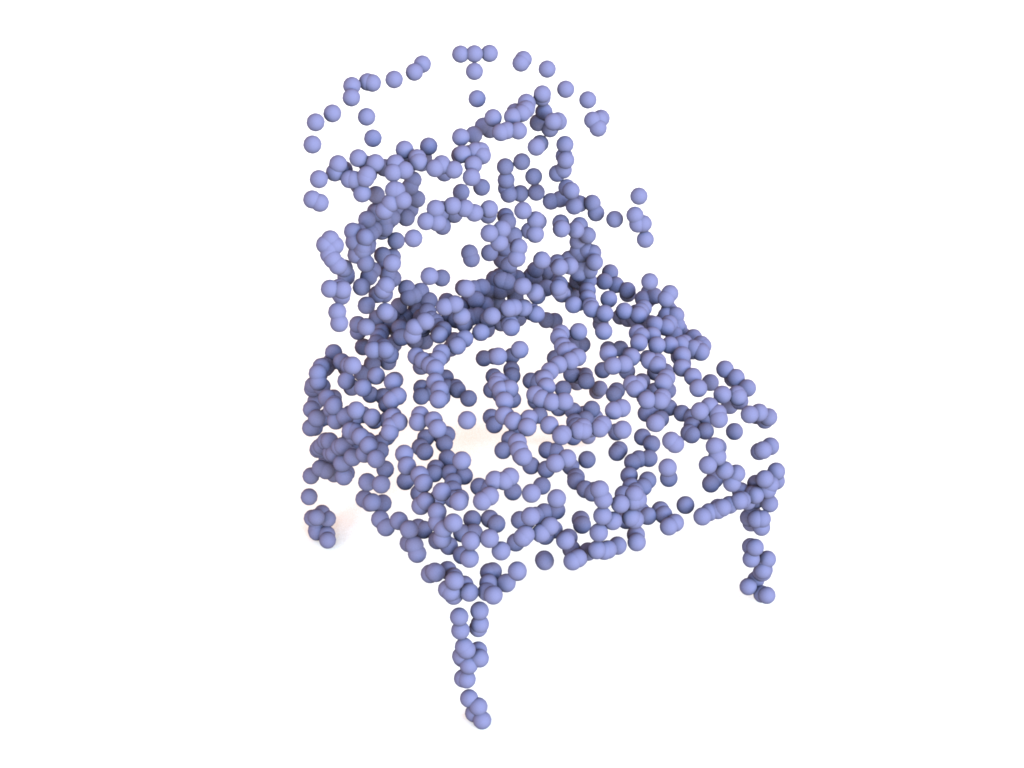}
		\label{fig:pc_target_3}
	}
	\subfigure[NN 3]
	{
		\includegraphics[width=0.14\textwidth]{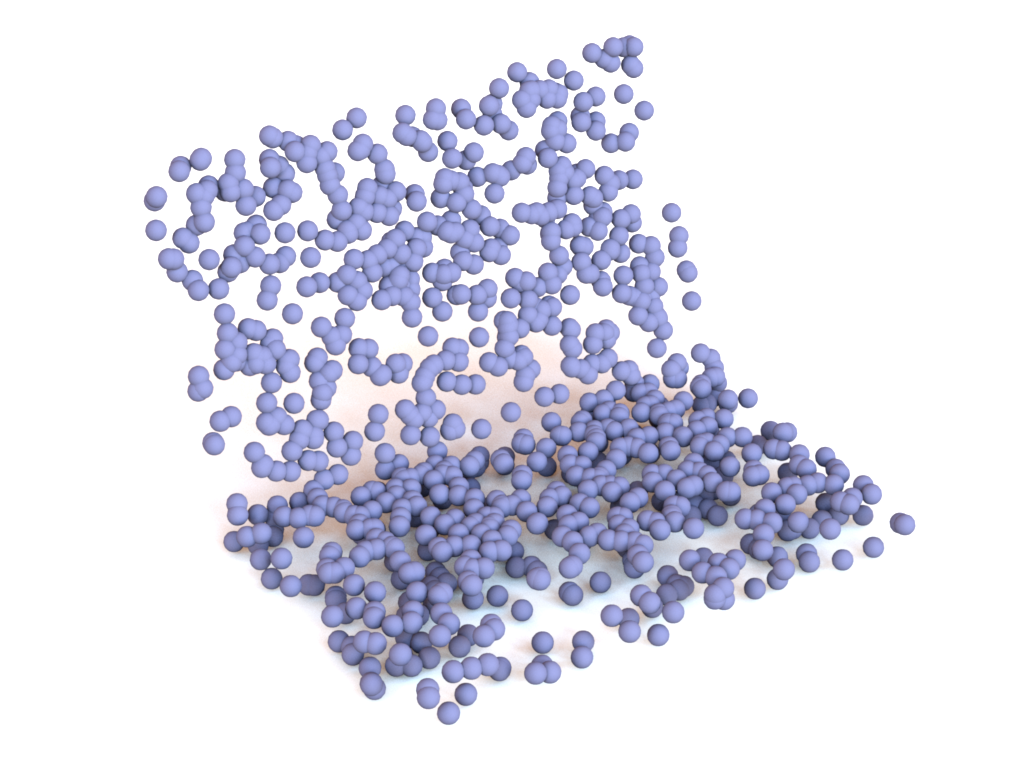}
		\label{fig:pc_nn_3}
	}
	\caption{Nearest neighbor search results on ModelNet40 point clouds.
	GT and DeepSets (m) indicate the ground-truth and DeepSets (max), respectively.
	Query and NN pairs are the query and its nearest neighbor examples found by our method.
	}
	\label{fig:pointcloud_retrieval}
\end{figure}

\begin{table}[t]
	\caption{Convergence quality and its execution time on nearest neighbor retrieval for ModelNet40 point clouds.
	All settings follow the settings in \figref{fig:pointcloud_retrieval}.
	\label{tab:effects_modelnet40}}
	\vspace{10pt}
	\centering
	\begin{tabular}{crrrr}
		\toprule
		& \multicolumn{2}{c}{\textbf{100 examples}} & \multicolumn{2}{c}{\textbf{Entire examples}} \\
		$L$ & \multicolumn{1}{c}{Chamfer distance} & \multicolumn{1}{c}{Time ($10^3$ sec.)} & \multicolumn{1}{c}{Chamfer distance} & \multicolumn{1}{c}{Time ($10^3$ sec.)} \\
		\midrule
		$1$ & $63.720 \pm 23.780$ & $1.133 \pm 0.049$ & $59.095 \pm 10.629$ & $1.071 \pm 0.067$\\
		$2$ & $59.079 \pm 25.610$ & $1.844 \pm 0.109$ & $61.783 \pm 10.000$ & $1.835 \pm 0.094$\\
		$5$ & $41.072 \pm 15.807$ & $6.568 \pm 0.485$ & $53.820 \pm \, \ 5.643$ & $6.456 \pm 0.372$\\
		$10$ & $38.422 \pm 12.807$ & $21.156 \pm 1.360$ & $49.798 \pm \, \ 3.121$ & $23.482 \pm 1.111$\\
		$20$ & $38.422 \pm 12.807$ & $104.330 \pm 5.045$ & $49.054 \pm \, \ 2.614$ & $97.179 \pm 3.882$\\
		\bottomrule
	\end{tabular}
\end{table}

\subsection{Active Nearest Neighbor Search for Point Clouds \label{subsec:exp_pointcloud}}

ModelNet40 dataset~\cite{WuZ2015cvpr} contains 40 categories of 12,311 3D CAD models.
Point cloud representation is obtained by sampling uniformly 1,024 points from the surface of each 3D model using Open3D~\cite{ZhouQ2018arxiv}.
Nearest neighbor search for point clouds requires large number of Chamfer distance calculations, 
which is a time-consuming task, in particular when the size of dataset is large.
We employ our Bayesian optimization over sets to actively select a candidate whose Chamfer distance from the query is to be computed, 
while the linear scan requires the calculation of Chamfer distance from the query to every other data in the dataset:
\begin{equation}
    d_{\textrm{Chamfer}}(\bX, \bY) = \sum_{\bx \in \bX} \min_{\by \in \bY} \| \bx - \by \|_2
\end{equation}
where $|\bX| = |\bY| = m$ and each element in $\bX$ and $\bY$ is a three-dimensional real vector.

For this experiment, we use one oracle and three additional baselines:

\noindent \textbf{Ground-truth (worst)}
It is the worst case to achieve the ground-truth.
The best retrieval case is to find the ground-truth at once,
and the usual case would be between the best and the worst.

\noindent \textbf{DeepSets ($\mu$) / DeepSets (max) / DeepSets ($+$)}
These baselines are implemented to embed point clouds to a single vector 
using DeepSets~\cite{ZaheerM2017neurips},
and measure $l^2$ distance between the embeddings.
Because we can access to class information of ModelNet40 dataset~\cite{WuZ2015cvpr}, 
these neural networks can be trained to match the information.
We use three fully-connected layers with batch normalization~\cite{IoffeS2015icml} as an instance-wise network,
and four fully-connected layers with batch normalization as a network after aggregation.
ReLU is employed as an activation function.
The choice of global aggregation methods determines each baseline:
(i) mean aggregation is $\mu$;
(ii) max aggregation is max;
and (iii) sum aggregation is $+$.
Because retrieving 1-nearest neighbor with these methods is hard to obtain the nearest neighbor,
we choose the nearest neighbor from 3-nearest neighbors to fairly compare with our methods.

Experiments with two different settings were carried out: 
(i) the size of point clouds is 100; 
(ii) full-size point clouds (the cardinality is 12,311).
In the case of DeepSets, point clouds are embedded into a low-dimensional Euclidean space, so that Euclidean distance is used to search a nearest neighbor (i.e., approximate search). On the other hand, our method actively selects a candidate gradually in the point cloud dataset.
The nearest neighbor determined by our method, given the query, 
is shown in \figref{fig:pointcloud_retrieval}.
About 20 iterations of the procedure is required to achieve the better performance, 
compared to DeepSets (\figref{fig:pc_100} and \figref{fig:pc_12311}).

\subsection{Empirical Analysis on Computational Cost\label{subsec:discussion}}

The computational costs from \tabref{tab:kernel_n_d_m_l} to \tabref{tab:effects_modelnet40} 
are presented as a function of $L$.
These results are measured using a native implementation of set kernels written in \texttt{Python}.
As mentioned in \secref{sec:method}, 
the computational costs follow our expectation, which implies
that the complexity for computing a covariance matrix over sets 
is the major computations in the overall procedure.

\section{Conclusion\label{sec:conclusion}}

In this paper, we propose the Bayesian optimization method over sets, 
which takes a set as an input and produces a scalar output.
Our method based on GP regression models a surrogate function using set-taking covariance functions, 
referred to as set kernel.
We approximate the set kernel to the efficient positive-definite kernel 
that is an unbiased estimator of the original set kernel.
To find a next set to observe, we employ a constrained acquisition function optimization 
using the symmetry of the feasible region defined over sets.
Moreover, we provide a simple analysis on cumulative regret bounds of our methods.
Our experimental results demonstrate our method can be used in some novel applications for Bayesian optimization.


%
\section*{Conflict of interest}

The authors declare that they have no conflict of interest.

\bibliographystyle{spmpsci}
\bibliography{sjc}

\end{document}